\documentclass[nohyperref]{article}
\usepackage{fullpage}
\usepackage{parskip}

\usepackage[utf8]{inputenc} %
\usepackage[T1]{fontenc}    %
\usepackage{url}            %
\usepackage{booktabs}       %
\usepackage{amsfonts}       %
\usepackage{natbib}
\usepackage{array}
\usepackage{mathtools}
\usepackage{graphicx}

\usepackage{hyperref}
\usepackage{algorithm}
\usepackage{shortcuts}

\newcommand{\csCB}{\text{ConfSet}_{\textsf{CB}}}
\newcommand{\cs}{\text{ConfSet}_{\textsf{RL}}}
\newcommand{\online}{\textsc{O-DISCO}}
\newcommand{\offline}{\textsc{P-DISCO}}
\newcommand{\cb}{\textsc{DistUCB}}
\newcommand{\measureSpace}{L^2(\lambda)}
\newcommand{\eluDim}{\op{dim}_{\ell_1\normalfont\textsf{E}}}
\newcommand{\distEluDim}{\op{dim}_{\ell_1\normalfont\textsf{DE}}}
\newcommand{\dimCB}{d_{\normalfont\textsf{CB}}}
\newcommand{\dimRL}{d_{\normalfont\textsf{RL}}}
\newcommand{\dimRLv}{d_{\normalfont\textsf{RL,V}}}
\newcommand{\gap}{\op{Gap}}
\newcommand{\vargap}{\op{VarGap}}
\newcommand{\CStarGap}{C^\star\op{Gap}}

\title{More Benefits of Being Distributional:\\Second-Order Bounds for Reinforcement Learning}
\newcommand{\addLink}[1]{\href{#1}{\nolinkurl{#1}}}
\usepackage{authblk}
\author[1]{Kaiwen Wang\thanks{Correspondence to \addLink{https://kaiwenw.github.io/}.}}
\author[1]{Owen Oertell}
\author[2]{Alekh Agarwal}
\author[1]{Nathan Kallus}
\author[1]{Wen Sun}
\affil[1]{Cornell University}
\affil[2]{Google Research}
\date{\today}

\begin{document}

\maketitle

\begin{abstract}
In this paper, we prove that Distributional Reinforcement Learning (DistRL), which learns the return distribution, can obtain second-order bounds in both online and offline RL in general settings with function approximation. Second-order bounds are instance-dependent bounds that scale with the variance of return, which we prove are tighter than the previously known small-loss bounds of distributional RL. To the best of our knowledge, our results are the first second-order bounds for low-rank MDPs and for offline RL. When specializing to contextual bandits (one-step RL problem), we show that a distributional learning based optimism algorithm achieves a second-order worst-case regret bound, and a second-order gap dependent bound, simultaneously. We also empirically demonstrate the benefit of DistRL in contextual bandits on real-world datasets.  We highlight that our analysis with DistRL is relatively simple, follows the general framework of optimism in the face of uncertainty and does not require weighted regression. Our results suggest that DistRL is a promising framework for obtaining second-order bounds in general RL settings, thus further reinforcing the benefits of DistRL.
\end{abstract}

\section{Introduction}\label{sec:intro}
The aim of reinforcement learning (RL) is to learn a policy that minimizes the expected cumulative cost along its trajectory. Typically, \emph{squared loss} is used in standard RL algorithms \citep{mnih2015human,haarnoja2018soft} for learning the value function, the expected cost-to-go from a given state.
As an alternative to squared loss, \citet{bellemare2017distributional} proposed to learn the \emph{whole conditional distribution} of cost-to-go with distributional loss functions such as the negative log-likelihood or the pinball loss \citep{dabney2018implicit}. This paradigm is aptly called Distributional RL (DistRL) and has since been empirically validated in numerous real-world tasks \citep{bellemare2020autonomous,bodnar2020quantile,fawzi2022discovering,dabney2018distributional}, as well as in benchmarks for both online \citep{yang2019fully} and offline RL \citep{ma2021conservative}.
However, there is a lack of understanding for why DistRL often attains stronger performance and sample efficiency \citep{lyle2019comparative}.

This raises a natural theoretical question: when and why is DistRL better than standard RL?
\citet{wang2023the} recently proved that DistRL based on maximum likelihood estimation (MLE) results in \emph{small-loss bounds}, which are instance-dependent bounds that scale with the minimum possible expected cumulative cost $V^\star$ for the task at hand. If the optimal policy makes few blunders on average, \ie, $V^\star\approx 0$, then small-loss bounds converge at the fast $\Ocal(1/N)$ rate, while standard RL bounds converge at a $\Ocal(1/\sqrt{N})$ rate which is worst-case in nature.

In this paper, we refine the analyses of \citet{wang2023the} and prove that DistRL actually attains tighter \emph{second-order bounds} in both online and offline settings.
Instead of scaling with $V^\star$ as in small-loss bounds, our second-order bounds scale with the variance of the policy's cumulative cost. In offline RL, it is the optimal policy's variance, whilst in online RL, it is the variance of policies played by the algorithm. In both cases, our second-order result is \emph{strictly tighter} than the previously known small-loss bounds (a.k.a. first-order bounds), \ie, second-order implies first-order bounds.
In particular, our second-order bounds yield fast $\Ocal(1/N)$ rates in near-deterministic tasks where $V^\star$ may still be far from zero. Our theory applies at the same generality as \citet{wang2023the}. Moreover, in contextual bandits (one-step RL), we prove a novel first and second-order gap-dependent bound that incorporates $V^\star$ and variance into the gap definition. Finally, in contextual bandits, we empirically show that our distributionally optimistic algorithm is efficiently implementable with neural networks via width computation \citep{feng2021provably} and outperforms the same optimistic algorithm with squared loss \citep{foster2018practical}.

Our contributions are summarized as follows:
\begin{enumerate}[leftmargin=0.7cm]
    \item For online RL, we show that DistRL enjoys second-order bounds in MDPs with low $\ell_1$-distributional eluder dimension \citep{wang2023the}.
    These are the first second-order bounds in MDPs with function approximation, \eg, low-rank MDPs
    (\cref{sec:online-rl}).
    \item For offline RL, we show that DistRL enjoys second-order bounds with single-policy coverage, the first of such bounds to our knowledge (\cref{sec:offline-rl}).
    \item For contextual bandits, our online algorithm further achieves a novel first/second-order gap-dependent bound (\cref{sec:gap-dependent-cb}). Finally, we empirically evaluate our distributional contextual bandit algorithm and show it outperforms the squared loss baseline (\cref{sec:experiments}).
\end{enumerate}

\section{Related Works}
\paragraph{Theory of DistRL.} \citet{rowland2018analysis,rowland2023analysis} showed that DistRL algorithms such as C51 and QR-DQN converges asymptotically with a tabular representation. This unfortunately does not imply finite-sample statistical improvements over standard RL, which is our focus. Recently, \citet{rowland2023statistical} showed that quantile temporal-difference (QTD) learning may have smaller bounded variance in each update step than temporal-difference (TD) learning, which may have unbounded variance.
While this finding may explain improved training stability, it does not affirmatively imply that QTD obtains better finite-sample regret, which is our focus.
For off-policy evaluation (OPE), \citet{wu2023distributional} showed that fitted likelihood estimation can learn the true return distribution up to errors in total variation and Wasserstein distance. We focus on online and offline RL rather than OPE.

\paragraph{Small-loss Bounds from DistRL.}
The closest work to ours is \citet{wang2023the} which showed that MLE-based DistRL can achieve small-loss bounds in online RL and offline RL under distributional Bellman completeness, building on the earlier contextual bandit results of~\citet{fosterKrishnamurthy}. While \citet{wang2023the} gave the first small-loss bounds in low-rank MDPs and in offline RL, we prove that their DistRL algorithms can actually achieve tighter, second-order bounds \emph{under identical assumptions}. Our bounds are strictly more general than small-loss (a.k.a. first-order) bounds as shown by the following theorem.
\begin{theorem}[Informal]\label{thm:informal-second-order-implies-small-loss}
In online and offline RL, a second-order bound implies a first-order bound (with a worse universal constant). This is formalized in \cref{thm:second-order-implies-small-loss}.
\end{theorem}

\paragraph{Other second-order bounds.}
Variance-dependent (a.k.a. second-order bounds) are known in tabular MDPs \citep{zanette2019tighter,zhou2023sharp,zhang2023settling,talebi2018variance},
linear mixture MDPs \citep{pmlr-v195-zhao23a}, and linear contextual bandits \citep{ito2020tight,olkhovskaya2023first}.
Many prior works use variance weighted regression but their analysis does not easily extend beyond linear function approximation.
Surprisingly, we show that by simply learning the return distribution with MLE, one can obtain general variance-dependent bounds, by leveraging the tool of triangular discrimination that was first leveraged in~\citet{fosterKrishnamurthy}. In other words, DistRL is an attractive alternative to variance weighted regression for obtaining sharp second-order bounds in RL.

\section{Preliminaries}\label{sec:prelim}
\paragraph{Contextual Bandits (CB).} We first consider CBs with context space $\Xcal$, finite action space $\Acal$ of size $A$ and normalized conditional costs $C:\Xcal\times\Acal\to\Delta([0,1])$,
where $\Delta([0,1])$ is the set of all distributions on $[0,1]$ that are absolutely continuous with respect to some dominating measure $\lambda$, \eg, Lebesgue for continuous or counting for discrete.
We identify such a distribution via its density with respect to $\lambda$, hence we write $(C(x,a))(y)$ or $C(y\mid x,a)$ for the density of $C(x,a)$ at $y$.
The CB proceeds over $K$ episodes as follows: at episode $k\in[K]=\{1,\dots,K\}$, the learner observes a context $x_k\in\Xcal$, takes an action $a_k\sim\Acal$, and receives a cost $c_k\sim C(x_k,a_k)$.
We do not require that contexts are sampled from a fixed distribution; they may be arbitrarily chosen by an adaptive adversary.
The goal is to minimize the regret, defined as
\begin{equation*}
    \op{Reg}_{\normalfont\textsf{CB}}(K) := \sum_{k=1}^K\bar C(x_k,a_k)-\min_{a\in\Acal}\bar C(x_k,a),
\end{equation*}
where the bar denotes the mean of the distribution, \ie, $\bar f =\int y f(y)\diff\lambda(y)$ for any $f\in\Delta([0,1])$. We'll also use $\Var(f)=\int (y-\bar f)^2f(y)\diff\lambda(y)$ to denote the variance.

\paragraph{Reinforcement Learning (RL).} We now consider a Markov Decision Process (MDP) with observation space $\Xcal$, finite action space $\Acal$ of size $A$, horizon $H$, transition kernels $P_h:\Xcal\times\Acal\to\Delta(\Xcal)$, and normalized cost distributions $C_h:\Xcal\times\Acal\to\Delta([0,1])$ at each step $h\in[H]$.
Given a policy $\pi:\Xcal\to\Delta(\Acal)$ and an initial state $x_1\sim\Xcal$, the ``roll in'' process occurs as follows: for each step $h=1, 2, \dots, H$, the policy $\pi$ samples an action $a_h$ based on the current state $x_h$, incurs a cost $c_h$ from the cost distribution, and transitions to the next state $x_{h+1}$.
The return is the cumulative cost from this random process $Z^\pi(x_1):=\sum_{h=1}^H c_h$.
The value is the expected return $V^\pi(x_1):=\EE[Z^\pi(x_1)]$. We use subscript $h$ to denote cost-to-go from a particular step: $Z^\pi_h(x_h)=\sum_{t=h}^Hc_t$ and $V^\pi_h(x_h)=\EE[Z^\pi_h(x_h)]$.
We use $Z^\star, V^\star$ to denote these quantities for the optimal, min-cost policy $\pi^\star$.
We use $Z^\pi_h(x_h,a_h)$ to denote the random cost-to-go conditioned on rolling in $\pi$ from $x_h,a_h$, and so $Q^\pi_h(x_h,a_h):=\EE[Z^\pi_h(x_h,a_h)]$.
Without loss of generality, we assume cumulative costs $\sum_{h=1}^Hc_h$ are normalized in $[0,1]$ almost surely, to avoid artificial scaling in $H$ \citep{jiang2018open}.

The \emph{Online RL} problem proceeds over $K$ episodes: at episode $k\in[K]$, the learner executes a policy $\pi^k:\Xcal\to\Delta(\Acal)$ from an initial state $x_{1,k}$. We do not require that $x_{1,k}$ are sampled from a fixed distribution; they may be chosen by an adaptive adversary.
The goal is to minimize regret,
\begin{equation*}
    \op{Reg}_{\textsf{RL}}(K) := \sum_{k=1}^K V^{\pi^k}(x_{1,k})-V^\star(x_{1,k}).
\end{equation*}

In \emph{Offline RL}, the learner is directly given \emph{i.i.d.} samples of transitions drawn from unknown distributions $\nu_1,\dots,\nu_H$, and the goal is to learn a policy with a lower cost than any other policy whose behavior is covered by the dataset, similar to prior best-effort guarantees in offline RL~\citep{liu2020provably, xie2021bellman}.
Concretely, the learner receives a dataset $\Dcal=(\Dcal_1, \Dcal_2, \dots, \Dcal_H)$, where each $\Dcal_h$ contains $N$ \emph{i.i.d.} samples $(x_{h,i},a_{h,i},c_{h,i},x_{h,i}')$ such that $(x_{h,i},a_{h,i})\sim\nu_h, c_{h,i}\sim C_h(x_{h,i},a_{h,i}), x_{h,i}'\sim P_h(x_{h,i},a_{h,i})$.
Unlike the online setting where initial states can be adversarial, we assume for offline RL that initial states are sampled from a fixed and known distribution $d_1$.

\paragraph{Distributional RL.}
The Bellman operator acts on a function $f:\Xcal\times\Acal\to[0,1]$ as follows:
$\Tcal_h^\pi f(x,a) = \bar C_h(x,a) + \Eb[x'\sim P_h(x,a),a'\sim\pi(x')]{f(x',a')}.$
Analogously, the distributional Bellman operator \citep{bellemare2017distributional} acts on a conditional distribution $d:\Xcal\times\Acal\to\Delta([0,1])$ as follows: $\Tcal_h^{\pi,D} d(x,a) \stackrel{D}{=} C_h(x,a) * d(x',a')$, where $x'\sim P_h(x,a), a'\sim \pi(x')$ and $*$ denotes convolution. Another sampling view of the distributional Bellman operator is that $z\sim \Tcal_h^{\pi, D} d(x,a)$ is equivalent to: $c \sim C_h(x,a), x'\sim P_h(x,a), a'\sim \pi(x'), y \sim d(x',a')$ and $z:=c+y$.
Also recall the optimality operator $\Tcal_h^\star$ and its distributional variant $\Tcal_h^{\star,D}$ are defined as follows:
$\Tcal_h^\star f(x,a) = \bar C_h(x,a) + \Eb[x'\sim P_h(x,a)]{\min_{a\in\Acal}f(x',a')}$ and $\Tcal_h^{\star,D}d(x,a)\stackrel{D}{=} C_h(x,a)+d(x',a')$ where $x'\sim P_h(x,a), a'=\argmin_a\bar d(x',a)$.

\paragraph{Triangular Discrimination.} For any distributions $f,g\in\measureSpace$, their triangular discrimination \citep{topsoe2000some} is defined as $D_\triangle(f\Mid g) := \int\frac{(f(y)-g(y))^2}{f(y)+g(y)}\diff\lambda(y)$, which is equivalent to the squared Hellinger distance up to universal constants.
Please see \cref{tab:notation} for an index of notations.

\section{Warmup: Second-Order Bounds for CBs}\label{sec:cb}

\begin{algorithm}[t!]
\caption{\cb{} (\online{} at $H=1$)}
\label{alg:online_cb}
\begin{algorithmic}[1]
    \State\textbf{Input:} no. episodes $K$, distribution class $\Fcal$
    \State Init $\Dcal_{0}\gets\emptyset$ and $\Fcal_0\gets\Fcal$.
    \For{episode $k=1,2,\dots,K$}
        \State Observe context $x_k$.
        \State Play $a_k = \argmin_{a\in\Acal}\min_{f\in\Fcal_{k-1}}\bar f(x_k,a)$.
        \State Observe cost $c_k\sim C(x_k,a_k)$.
        \State $\Dcal_k\gets\Dcal_{k-1}\cup\braces{(x_k,a_k,c_k)}$, $\Fcal_k\gets\csCB(\Dcal_k)$.
    \EndFor
\end{algorithmic}
\end{algorithm}

As a warmup, we consider contextual bandits and prove that distributional UCB (\cb{}) attains second-order regret.
The distributional confidence set is the main construct that is optimized over to ensure optimism.
To construct it, we need a dataset of state, action, costs, $D=\braces*{x_{i},a_{i},c_{i}}_{i\in[N]}$, a threshold $\beta$ to be specified later, as well as a function class $\Fcal\subset\Xcal\times\Acal\to\Delta([0,1])$ containing the true conditional cost distribution $C(\cdot\mid x,a)$.
Then, the confidence set is
\begin{equation*}
\csCB(D)=\Big\{ f\in\Fcal:\Lcal_{\textsf{CB}}(f,D)\geq\max_{g\in\Fcal}\Lcal_{\textsf{CB}}(g,D)-\beta \Big\},
\end{equation*}
where $\Lcal_{\textsf{CB}}(f,D):=\sum_{i=1}^N\log f(c_i\mid x_i,a_i)$ is the log-likelihood of $f$ on $D$. In words, $\csCB(\Fcal,D)$ contains all functions $f\in\Fcal$ that are $\beta$-near-optimal according to the log-likelihood.
Then, \cb{} simply selects the action with the minimum lower confidence bound (LCB) induced by the current confidence set.
\begin{restatable}{theorem}{secondOrderCB}\label{thm:second-order-cb}
Suppose $C\in\Fcal$ (realizability). For any $\delta\in(0,1)$, w.p. at least $1-\delta$, running \cb{} with $\beta=\log(K|\Fcal|/\delta)$ enjoys the regret bound,
\begin{align*}
    \op{Regret}_{\normalfont\textsf{CB}}(K)
    &\leq\wt\Ocal\Big(\sqrt{ \dimCB\beta\cdot \sum_{k=1}^K \Var(C(x_k,a_k)) }+\dimCB\beta\Big),
\end{align*}
where $\dimCB$ is the $\ell_1$-eluder dimension \citep{liu2022partially} of $\braces{(x,a)\mapsto D_\triangle(f(x,a)\Mid C(x,a)): f\in\Fcal}$ at  threshold $K^{-1}$. This is a special case of the distributional eluder dimension (\cref{def:dist-eluder}) where $\mathfrak{D}=\{\delta_{z}:z\in\Xcal\times\Acal\}$.
\end{restatable}
The dominant term scales with $\sqrt{ \sum_{k=1}^K \Var(C(x_k,a_k)) }$ which is sharper than the $\sqrt{K}$ bound of RegCB \citep{foster2018practical}, the squared loss variant of \cb{}.
For example, in deterministic settings, our variance-dependent regret scales as $\wt\Ocal(\dimCB)$, which is tight in $K$ up to log factors.
Nonetheless, confidence-set based strategies like \cb{} and RegCB are not minimax-optimal as the eluder dimension may scale linearly in $\Fcal$ \citep[Proposition 1]{foster2018practical}. It would be interesting to derive second-order regret with inverse-gap weighting \citep{foster2020beyond}.

\paragraph{Practical considerations.} We note that \cb{} is amenable to practical implementation since conditional on $x_k$ and $a$, the LCB can be computed efficiently via binary search \citep{foster2020beyond} or disagreement computation \citep{feng2021provably}. We include implementation pseudo-code and empirical results in \cref{sec:experiments} and the Appendix.

\subsection{Proof of \cref{thm:second-order-cb}}\label{sec:proof-of-second-order-cb}
Our first step is to bound the difference of means by variances multiplied by the triangular discrimination.
\begin{restatable}{lemma}{varKeyIneqOne}
For $f,g\in\measureSpace$ s.t. $D_\triangle(f\Mid g)\leq \frac12$,
\begin{equation}
    \abs{ \bar f-\bar g }\leq 2\sqrt{ (\Var(f)+\Var(g)) D_\triangle(f\Mid g) }. \label{eq:var-key-ineq1}
\end{equation}
\end{restatable}
This lemma tightens Eq.($\Delta_1$) of \citet{wang2023the} so that variances of $f$ and $g$ appear in the RHS instead of the means. Note that Eq.($\Delta_1$) of \citet{wang2023the} holds unconditionally, while our lemma requires $D_\triangle(f\Mid g)\leq\frac12$ which is absorbed in the lower order term of the next lemma. This lower order term is a key reason we need the bounded eluder dimension assumption.
\begin{restatable}{lemma}{varKeyIneqTwo}
For any $f,g\in\measureSpace$, we have
\begin{equation}
    \abs{ \bar f-\bar g }\leq 4\sqrt{ \Var(f) D_\triangle(f\Mid g) } + 5D_\triangle(f\Mid g). \label{eq:var-key-ineq2}
\end{equation}
\end{restatable}
We now bound the regret in a standard way with optimism, \ie, w.h.p. $\bar f_k(x_k,a_k)\leq \min_a\bar C(x_k,a)$, which is ensured by optimizing the confidence set.
Let $\delta_k(x,a):=D_\triangle(f_k(x,a)\Mid C(x,a))$.
Then,
{\small\begin{align*}
    &\sum_{k=1}^K \bar C(x_k,a_k)-\min_a\bar C(x_k,a)
    \\\leq&\sum_{k=1}^K \bar C(x_k,a_k)-\bar f_k(x_k,a_k) \tag{optimism}
    \\\leq&\sum_{k=1}^K 4\sqrt{\Var(C(x_k,a_k))\delta_k(x_k,a_k)}+5\delta_k(x_k,a_k) \tag{\cref{eq:var-key-ineq2}}
    \\\leq&4\sqrt{\sum_{k=1}^K \Var(C(x_k,a_k)) \Delta}+5\Delta, \tag{Cauchy-Schwarz}
\end{align*}
}
where $\Delta=\sum_{k=1}^K\delta_k(x_k,a_k)$.
Finally, using MLE generalization bound and the fact that $f_k \in \mathcal{F}_{k-1}$, with probability at least $1-\delta$, we have for all $k\in[K]$: $\sum_{i=1}^{k-1}\delta_k(x_i,a_i)\leq \log(|\Fcal|K/\delta)$ \citep[Lemma E.3]{wang2023the}.
Thus, applying pigeon-hole argument of eluder dimension gives $\Delta\leq 4\dimCB(1/K)\log(|\Fcal|K/\delta)\log(K)$ \citep[Proposition 21]{liu2022partially}.
This concludes the proof.

\subsection{First and Second-Order Gap-Dependent Bounds}\label{sec:gap-dependent-cb}
While it is known that UCB attains gap-dependent bounds, here we prove \emph{first and second-order gap-dependent bounds} which are novel to the best of our knowledge.
Recall that the gap at context $x$ and action $a$ is defined as $\gap(x,a):=\bar C(x,a) - \min_{a^\star\in\Acal}\bar C(x,a^\star)$. We define our novel first and second-order min-gaps as follows:
\begin{align*}
    &\CStarGap=\min_{x\in\Xcal} \min_{\substack{a\in\Acal:\gap(x,a)>0\\\wedge \min_{a^\star}\bar C(x,a^\star)>0}} \frac{\gap(x,a)}{\min_{a^\star}\bar C(x,a^\star)},
    \\&\vargap=\min_{x\in\Xcal} \min_{\substack{a\in\Acal:\gap(x,a)>0\\\wedge \Var(C(x,a))>0}} \frac{\gap(x,a)}{\sqrt{\Var(C(x,a))}}.
\end{align*}
The inner min is taken to be $\infty$ if the condition is empty.
\begin{restatable}{theorem}{CBGapDependentRegret}\label{thm:gap-dependent-regret}
Assume the premise of \cref{thm:second-order-cb}.
If $\max\prns{\vargap,\CStarGap}\geq\frac{1}{\sqrt{K}}$, then
\begin{align*}
\op{Regret}_{\normalfont\textsf{CB}}(K)\leq\wt\Ocal\prns{\dimCB\beta + \dimCB\beta\min\braces{\vargap^{-1}, \CStarGap^{-1}} }.
\end{align*}
\end{restatable}
As usual, we have a $\gap^{-1}$-type bound that implies $\Ocal(\dimCB\log K)$ regret when the gap is large. Our key innovation lies in the definition of $\CStarGap$ and $\vargap$, which are inversely weighted by the optimal mean cost or variance of each context. Our weighted min-gaps are always larger than the standard min-gap (since $\bar C(x,a),\Var(C(x,a))\leq 1$) but they can be much larger in small-loss or near-deterministic regimes. We note that \cb{}'s regret is simultaneously bounded by \emph{both} \cref{thm:gap-dependent-regret} and \cref{thm:second-order-cb} under the same hyperparameters.

\section{Second-Order Bounds for Online DistRL}\label{sec:online-rl}
In this section, we show that the optimistic DistRL algorithm of \citet{wang2023the} actually enjoys second-order regret and PAC guarantees, which are strictly tighter than the previously known first-order bounds.
We first recall the MLE-confidence set for DistRL which generalizes $\csCB$ from the warmup.
Let $\Fcal$ be a set of conditional distributions, \ie, $(f_1,\dots,f_H)\in\Fcal$ where $f_h:\Xcal\times\Acal\to\Delta([0,1])$, which are candidate functions to fit $Z^\star$ or $Z^\pi$ (depending on the type of Bellman operator used) with MLE.
Given a dataset of state, action, cost, next state tuples, $D=\braces*{x_{h,i},a_{h,i},c_{h,i},x'_{h,i}}_{h\in[H],i\in[N]}$, and a distributional Bellman operator $\Tcal^D$, the MLE-confidence set is defined as
\begin{align*}
    \cs(D;\Tcal^D)=\bigg\{
    &f\in\Fcal\,:\,\forall h\in[H],
    \\&\Lcal_{\textsf{RL}}(f,D)\geq\max_{g\in\Fcal_h}\Lcal_{\textsf{RL}}(g,D)-\beta\bigg\},
\end{align*}
where $\Lcal_{\textsf{RL}}(f,D) := \sum_{i=1}^N\log f_h(z^f_{h,i}\mid x_{h,i},a_{h,i})$ and $z_{h,i}^f\sim\Tcal^{D}_h f_{h+1}(x_{h,i},a_{h,i})$.
In words, $\cs(D;\Tcal^D)$ contains all functions $f\in\Fcal$ such that \emph{for all $h\in[H]$}, $f$ is $\beta$-near-optimal w.r.t. the MLE loss for solving $f_h\approx\Tcal^D_h f_{h+1}$.
Since this construction happens in a TD fashion, a standard condition called distributional Bellman Completeness (BC) is needed to guarantee that MLE succeeds for all $h\in[H]$ \citep{wu2023distributional,wang2023the}.
\begin{assumption}[Bellman Completeness]\label{asm:bellman-completeness}
For all $\pi,h\in[H]$, $f_{h+1}\in\Fcal_{h+1}\implies\Tcal^{\pi,D}_hf_{h+1}\in\Fcal_h$.
\end{assumption}
BC is a standard assumption in model-free online and offline RL; without it, TD and fitted-Q can diverge or converge to bad fixed points \citep{tsitsiklis1996analysis,munos2008finite,kolter2011fixed}.
As discussed in \citep{jin2021bellman,wang2023the}, the BC condition can be relaxed to ``generalized completeness'', \ie, there exist function classes $\Gcal_h$ such that $f_{h+1}\in\Fcal_{h+1} \implies \Tcal^{\pi,D}_hf_{h+1}\in\Gcal_h$.

\begin{algorithm}[t!]
\caption{\online{} \citep{wang2023the}}
\label{alg:onlinerl}
\begin{algorithmic}[1]
    \State\textbf{Input:} no. episodes $K$, distribution class $\Fcal$, \textsc{UAE} flag.
    \State Init $\Dcal_{h,0}\gets\emptyset$ for all $h\in[H]$ and $\Fcal_0\gets\Fcal$.
    \For{episode $k=1,2,\dots,K$}
        \State Observe init state $x_{1,k}$.
        \State Set $f^{(k)}\gets\argmin_{f\in\Fcal_{k-1}} \min_a \bar f_1(x_{1,k},a)$.
        \State For each $h$, set $\pi^k_h(x) = \argmin_a\bar f^{(k)}_{h}(x,a)$.
        \If{not \textsc{UAE}}
        \State \multiline{Run $\pi^k$ from $x_{1,k}$ and get trajectory $x_{1,k},a_{1,k},c_{1,k},..,x_{H,k},a_{H,k},c_{H,k}$.
        Then, $\forall h$,
        $\Dcal_{h,k}=\Dcal_{h,k-1}\cup\braces{(x_{h,k},a_{h,k},c_{h,k},x_{h+1,k})}$.
        }
        \Else
        \State \multiline{For each $h\in[H]$, roll in $\pi^k$ from $x_{1,k}$ for $h$ steps and take a random action, \ie, $x_{h,k}\sim d^{\pi^k}_h$, $a_{h,k}\sim\op{Unif}(\Acal)$, $c_{h,k}\sim C_h(x_{h,k},a_{h,k})$, $x_{h,k}'\sim P_h(x_{h,k},a_{h,k})$. Then, $\Dcal_{h,k}=\Dcal_{h,k-1}\cup\braces{(x_{h,k},a_{h,k},c_{h,k},x_{h,k}')}$.}
        \EndIf
        \State Update $\Fcal_k\gets\cs((\Dcal_{h,k})_{h\in[H]};\Tcal^{\star,D})$.
    \EndFor
    \State \textbf{Output:} $\bar\pi = \op{unif}(\pi^{1:K})$.
\end{algorithmic}
\end{algorithm}

Then, the \online{} algorithm of \citet{wang2023the} proceeds by selecting the optimistic $f^{(k)}$ in the confidence set $\Fcal_k$ at each round and playing the greedy policy $\pi^k$ w.r.t. $f$, where the ``playing'' can be done with uniform action exploration (\textsc{UAE}).
If \textsc{UAE}=\textsc{True}, then for each $h$, $\pi^k$ is rolled in for $h$ timesteps and takes a uniform action before the transition tuple is added to the dataset. Note that this requires $H$ rollouts per round but is necessary to capture general MDPs such as low-rank MDPs \citep{agarwal2020flambe}.

Finally, we adopt the $\ell_1$-distributional eluder dimension $(\distEluDim)$ defined as follows \citep{wang2023the}.
\begin{definition}[$\ell_p$-distributional eluder dimension]\label{def:dist-eluder}
Let $\Scal$ be any set, $\Psi$ be a set of functions of type $\Scal\to\RR$, and $\mathfrak{D}$ is a set of distributions over $\Scal$. For any $\eps_0\in\RR_+$, the $\ell_p$-distributional eluder dimension ($\dim_{\ell_p,\textsf{DE}}(\Psi,\mathfrak{D},\eps_0)$) is the length $L$ of the longest sequence $d^{(1)},..,d^{(L)}\subset\mathfrak{D}$ s.t. $\exists\eps\geq\eps_0, \forall t\in[L], \exists f\in\Psi$ where $\abs{\EE_{d^{(t)}}f}>\eps$ and also $\sum_{i=1}^{t-1}\abs{\EE_{d^{(i)}}f}^p\leq{\eps}^p$.
\end{definition}
We work with the same eluder dimensions for RL as in \citet{wang2023the} which employs the following:
\begin{align*}
    &\Psi_h=\{(x,a)\mapsto D_\triangle(f_h(x,a)\Mid\Tcal_h^{\star,D}f_{h+1}(x,a)): f\in\Fcal\},
    \\&\mathfrak{D}_h=\{(x,a)\mapsto d^\pi_h(x,a):\pi\in\Pi\}.
\end{align*}
Then, the $Q$-type RL dimension is
\begin{equation*}
    \dimRL(\eps):=\max_h \distEluDim(\Psi_h,\mathfrak{D}_h,\eps).
\end{equation*}
The V-type dimension $\dimRLv$ is analogous with $\Psi_{\textsf{V},h}=\braces{x\mapsto \EE_{a\sim\op{Unif}(\Acal)}[D_\triangle(f_h(x,a)\Mid\Tcal_h^{\star,D}f_{h+1}(x,a))]: f\in\Fcal}$.
As with $\dimCB$ (from the CB warmup), the threshold $\eps$ is taken as $1/K$ if none is provided.
We are now ready to state our online RL result.
\begin{restatable}[Second-order bounds for Online RL]{theorem}{secondOrderRL}\label{thm:second-order-rl}
Under \cref{asm:bellman-completeness},
for any $\delta\in(0,1)$, w.p. at least $1-\delta$, running $\online{}$ with $\beta=\log(HK|\Fcal|/\delta)$ enjoys,
{\small\begin{align*}
    \op{Reg}_{\normalfont\textsf{RL}}(K)\leq\wt\Ocal\prns{ H\sqrt{ \sum_{k=1}^K\Var(Z^{\pi^k}(x_{1,k}))\cdot \dimRL\beta} + H^{2.5}\dimRL\beta}.
\end{align*}}
If \textsc{UAE=True}, then the learned mixture policy $\bar\pi$ enjoys the PAC bound: w.p. at least $1-\delta$, $K\prns*{V^{\bar\pi}-V^\star}$ is at most,
{\small
\begin{align*}
    \wt\Ocal\Bigg(
    H\sqrt{ A\sum_{k=1}^K\Var(Z^{\pi^k}(x_{1,k}))\dimRLv\beta} +
    AH^{2.5}\dimRLv\beta \Bigg).
\end{align*}}
\end{restatable}
Compared to prior worst-case bounds for GOLF \citep{jin2021bellman} and small-loss bounds for \online{} \citep{wang2023the}, our new bound has one key improvement: the leading $\sqrt{K}$ terms are replaced by the square root of the sum of return variances $\sum_k\Var(Z^{\pi^k}(x_{1,k}))$.
The function class complexity measure $\log|\Fcal|$ can be generalized to bracketing entropy as in \citet{wang2023the}.
As \cref{thm:informal-second-order-implies-small-loss} shows, our second-order bounds are more general than the first-order bounds of \citet{wang2023the}.
For example, in deterministic MDPs where variance is zero, our second-order bound converges at a fast $\wt\Ocal(1/K)$ rate which is tight up to $\log K$ factors \citep{wen2017efficient}.
In contrast, $V^\star$ may be non-zero in which case the first-order bound converges at a slow $\wt\Omega(1/\sqrt{K})$ rate.

It may be surprising that DistRL actually helps for near-deterministic systems. This is because the agent does not \emph{a priori} know that the system is deterministic but a DistRL agent can quickly learn and adapt to this fact, while standard squared loss agents learn to adapt at a slower rate.
We highlight that our second-order bound comes easily from $D_\triangle$ generalization bounds of MLE; we do not need any variance weighted regression which almost all prior works to obtain second-order bounds and is hard to extend beyond linear function approximation.

Compared to variance weighted regression, one drawback of our DistRL approach (and other TD-style DistRL algorithms \citep{wu2023distributional}) is the requirement of a stronger, distributional completeness assumption (\cref{asm:bellman-completeness}), as well as a higher statistical complexity of $\Fcal$ (it is a class of conditional distributions rather than functions). Nevertheless, the empirical success of DistRL suggest these stronger conditions are likely satisfied in practice and the faster second-order rates may indeed offset the increased function class complexity.

\subsection{On low-rank MDPs.}
Low-rank MDPs \citep{agarwal2020flambe} are the standard model for non-linear representation learning in RL \citep{uehara2021representation,zhang2022making,ren2023spectral,chang2022learning}, and are defined as follow.
\begin{definition}[Low-Rank MDP]
An MDP is has rank $d$ if each step's transition has a low-rank decomposition $P(x'\mid x,a)=\phi^\star_h(x,a)^\top\mu^\star_h(x')$ where $\phi^\star_h(x,a),\mu^\star_h(x')\in\RR^d$ are unknown features that satisfy $\sup_{x,a}\|\phi^\star_h(x,a)\|_2\leq 1$ and $\|\int g\diff\mu^\star_h(s')\|\leq\|g\|_\infty\sqrt{d}$ for all $g:\Xcal\to\RR$.
\end{definition}
Our \cref{thm:second-order-rl} (with \textsc{UAE}) applies to low-rank MDPs the same way as \citet[Theorem 5.5]{wang2023the}. In particular, \citet{wang2023the} showed three important facts for rank-$d$ MDPs:
(i) the V-type eluder is controlled $\dimRLv(\eps)\leq \Ocal(d\log(d/\eps))$,
(ii) given a realizable $\Phi$ class, the linear function class $\mathcal{F}^{\text{lin}}=\prod_h\mathcal{F}_h^{\text{lin}}$ defined as
\begin{align*}
\mathcal{F}^{\text{lin}}_h =
\Big\{
&f(z\mid x,a)=\phi(x,a)^\top w(z): \phi\in\Phi\,,
\\&w:[0,1]\to\RR^d, \text{ s.t.}, \max_z\|w(z)\|_2\leq\sqrt{d}\Big\}
\end{align*}
satisfies distributional BC (\cref{asm:bellman-completeness}),
and (iii) if costs are discrete in a uniform grid of $M$ points, the bracketing entropy of $\mathcal{F}^{\text{lin}}$ is $\wt\Ocal(dM+\log|\Phi|)$.
Combining these facts with \cref{thm:second-order-rl} implies a second-order PAC bound for low-rank MDPs:
\begin{corollary}[Second-Order PAC Bound for Low-Rank MDPs]\label{cor:low-rank-mdp-online}
Suppose the MDP has rank $d$, assume $\phi^\star\in\Phi$ and costs are discrete in a uniform grid of $M$ points, then, w.h.p., \online{} with \textsc{UAE}, $\Fcal=\mathcal{F}^{\text{lin}}$ and $\beta=dM+\log(|\Phi|/\delta)$ outputs a policy $\bar\pi$ that satisfies,
\begin{align*}
V^{\bar\pi}-V^\star\leq \wt\Ocal\prns{ H\sqrt{ \frac{\overline{\Var_{1:K}}\cdot Ad\beta}{K} } + \frac{AdH^{2.5}\beta}{K} },
\end{align*}
where $\overline{\Var_{1:K}}=\frac1K\sum_{k=1}^K\Var(Z^{\pi^k}(x_{1,k}))$.
\end{corollary}
To the best of our knowledge, this is the first variance-dependent bound in RL beyond linear function approximation, which is a significant statistical benefit of DistRL.

\subsection{Proof Sketch for \cref{thm:second-order-rl}}
The new RL tool we'll employ is the following change-of-measure lemma for variance.
\begin{restatable}[Change of Variance]{lemma}{mainPaperChangeOfVarianceLemma}
For any $f:\Xcal\times\Acal\to\Delta([0,1])$, $\pi$ and $x_1$, we have
{\small\begin{align}
    &\EE_{\pi,x_1}\bracks{\Var(f_h(x_h,a_h))}\leq 2e\Var(Z^\pi(x_1))\,+ \nonumber
    \\&12H^2\EE_{\pi,x_1}\Big[\textstyle\sum_{t\geq h} D_\triangle(f_t(x_t,a_t)\Mid \Tcal^{\pi,D}_tf_{t+1}(x_t,a_t))\Big]. \label{eq:mainPaperChangeOfVarianceLemma}
\end{align}}
\end{restatable}
For each episode $k$, by optimism of $\bar f^{(k)}_1$, performance difference lemma and the fact $\Tcal_h^{\pi^k} \bar f_{h+1}^{(k)}(x_h,a_h)=\overline{\Tcal_h^{\pi^k} f_{h+1}^{(k)}}(x_h,a_h)$, we have
{\small\begin{align*}
    &V^{\pi^k}(x_{1,k})-V^\star(x_{1,k})
    \leq V^{\pi^k}(x_{1,k})-\min_a\bar f_1(x_{1,k},a)
    \\&=\sum_{h=1}^H\EE_{\pi^k,x_{1,k}}\bracks{ \overline{\Tcal_h^{\pi^k} f_{h+1}^{(k)}}(x_h,a_h)-\bar f_h^{(k)}(x_h,a_h) }.
\end{align*}}
Let $\delta_{h,k}(x,a):=D_\triangle(f^{(k)}_h(x,a)\Mid \Tcal^{\star,D}_hf^{(k)}_{h+1}(x,a))$.
{\small
\begin{align*}
    &\sum_{h=1}^H\EE_{\pi^k,x_{1,k}}\bracks{ \overline{\Tcal_h^{\pi^k} f_{h+1}^{(k)}}(x_h,a_h)-\bar f_h^{(k)}(x_h,a_h) }
    \\\leq&\sum_{h=1}^H 4\sqrt{\EE_{\pi^k,x_{1,k}}[\Var(f_h^{(k)}(x_h,a_h))]\cdot \EE_{\pi^k,x_{1,k}}[\delta_{h,k}(x_h,a_h)] }
    \\&\quad\,+ 5\EE_{\pi^k,x_{1,k}}[\delta_{h,k}(x_h,a_h)] \tag{\cref{eq:var-key-ineq2}}
    \\\leq&\sum_{h=1}^H 4\sqrt{ \prns{2e\Var(Z^\pi(x_{1,k})) + 12H^2\Delta_k} \cdot \EE_{\pi^k,x_{1,k}}[\delta_{h,k}(x_h,a_h)] }
    \\&\quad\,+ 5\EE_{\pi^k,x_{1,k}}[\delta_{h,k}(x_h,a_h)] \tag{\cref{eq:mainPaperChangeOfVarianceLemma}}
    \\\leq& 4\sqrt{ \prns{2e\Var(Z^\pi(x_{1,k})) + 12H^2\Delta_k} \cdot H\Delta_k} + 5H\Delta_k, \tag{Cauchy-Schwarz}
\end{align*}}
where $\Delta_k:= \sum_{h=1}^H\EE_{\pi^k,x_{1,k}}[\delta_{h,k}(x_h,a_h)]$.
Finally, we can sum over all episodes and use the fact that $\sum_k\Delta_k\leq Hd\log K$ w.p. $1-\delta$, where $d$ is the appropriate distributional eluder dimension depending on \textsc{UAE}. This last step is true due to MLE's generalization bound and standard eluder-type arguments from \citet{wang2023the}.

\section{Second-Order Bounds for Offline DistRL}\label{sec:offline-rl}

We now turn to offline RL and prove that pessimism in the face of uncertainty with MLE-confidence sets enjoys second-order PAC bounds under single-policy coverage.
The algorithm we study is \offline{} \citep{wang2023the}, which adapts the pessimism-over-confidence-set approach from BCP \citep{xie2021bellman} with the DistRL confidence set. As shown in \cref{alg:offlinerl}, \offline{} returns the best policy with respect to its pessimistic value estimate, induced by the distributional confidence set constructed with the given data.

Following recent advancements in offline RL \citep{xie2021bellman,uehara2022pessimistic,jin2021pessimism}, we prove best-effort guarantees that aim to compete with any covered comparator policy $\wt\pi$ and that only requires weak single-policy coverage. In particular, we do not suffer the strong all-policy coverage condition used in \citep{chen2019information}.
Recall the single-policy concentrability w.r.t. the comparator policy $\wt\pi$ is defined as $C^{\wt\pi}:=\max_h\|\diff d^{\wt\pi}_h/\diff\nu_h\|_\infty$.
We now state our main result for offline RL.

\begin{algorithm}[t!]
\caption{\offline{} \citep{wang2023the}}
\label{alg:offlinerl}
\begin{algorithmic}[1]
    \State\textbf{Input:} datasets $\Dcal_1,\dots,\Dcal_H$, distribution class $\Fcal$, policy class $\Pi$.
    \State $\forall\pi\in\Pi$, set $\Fcal_\pi\gets\cs((\Dcal_h)_{h\in[H]};\Tcal^{\pi,D})$.
    \State $\forall\pi\in\Pi$, set $f^\pi\gets\argmax_{f\in\Fcal_\pi}\EE_{x_1\sim d_1}[\bar f_1(x_1,\pi)]$.
    \State \textbf{Output:} $\wh\pi = \argmin_{\pi\in\Pi}\EE_{x_1\sim d_1}[\bar f_1^{\pi}(x_1,\pi)]$.
\end{algorithmic}
\end{algorithm}
\begin{restatable}[Second-order bounds for Offline RL]{theorem}{secondOrderOfflineRL}\label{thm:second-order-offline-rl}
Under \cref{asm:bellman-completeness}, for any $\delta\in(0,1)$, w.p. at least $1-\delta$, running \offline{} with $\beta=\log(H|\Pi||\Fcal|/\delta)$ learns a policy $\wh\pi$ that enjoys the following bound: for any comparator $\wt\pi\in\Pi$ (not necessarily the optimal $\pi^\star$), we have
\begin{equation*}
    V^{\wh\pi}-V^{\wt\pi} \leq \Ocal\prns{ H\sqrt{\frac{\Var(Z^{\wt\pi})C^{\wt\pi}\beta}{N}} + \frac{H^{2.5}C^{\wt\pi}\beta}{N} }.
\end{equation*}
\end{restatable}
Here, the leading term scales with the variance of the \emph{comparator policy's} returns $\Var(Z^{\wt\pi})$.
Since the variance is bounded by the first moment, this bound immediately improves the small-loss PAC bound of \citet{wang2023the}.
In near-deterministic settings, our second-order bound guarantees a fast $1/N$ rate and is tight up to log factors, which is not necessarily the case for small-loss bounds. In particular, our result shows that DistRL is even more robust to poor coverage than as shown in \citet{wang2023the}; that is, \offline{} can strongly compete with a comparator policy $\wt\pi$ if one of the following is true: (i) $\nu$ has good coverage over $\wt\pi$,
so the $\sqrt{1/N}$ term has a small constant; or (ii) $\nu$ has bad (but finite) coverage and $\wt\pi$ has small \emph{variance}, in which case we can still obtain a fast $1/N$ rate (with constant scaling with coverage).
To the best of our knowledge, this is the first second-order bound for offline RL.
\paragraph{Variance of $Z(\pi^k)$ vs. $Z(\pi^\star)$.}
In online RL, \cref{thm:second-order-rl,cor:low-rank-mdp-online} has the average variance of the played policies $Z(\pi^k)$, while in offline RL, \cref{thm:second-order-offline-rl} has the variance of the optimal policy $Z(\pi^\star)$ (if comparing with optimal policy). From a technical perspective, this dichotomy arises from the fact that in offline RL, single-policy concentrability allows us to change measure to $\pi^\star$, while in online RL, we cannot perform the switch and instead rely on eluder-type arguments.
The variances of $Z(\pi^k)$ and $Z(\pi^\star)$ are in general incomparable. Nonetheless, both statements are sharper than the small-loss bound as shown by \cref{thm:second-order-implies-small-loss}. Both are also tight in deterministic settings.

\paragraph{Computational Efficiency.}
Both \online{} and \offline{} optimize over the confidence set to ensure optimism and pessimism, respectively, but this step is known to be computationally hard even in tabular MDPs \citep{dann2018oracle}. This is also an issue for other version space algorithms: OLIVE \citep{jiang2017contextual}, GOLF \citep{jin2021bellman}, and BCP \citep{xie2021bellman}. However, the confidence set is needed for the purpose of deep exploration and can be replaced by myopic strategies such as $\eps$-greedy that are computationally cheap \citep{dann2022guarantees}. Finally, in the sequel, we show that in the case of CBs ($H=1$), \online{} can be efficiently implemented with neural nets via disagreement computation \citep{feng2021provably}.

\section{Contextual Bandit Experiments}\label{sec:experiments}
We empirically validate our stronger theory in the contextual bandit setting where our algorithm \cb{} can be efficiently implemented.
We demonstrate that learning the cost distribution (as in \cb{}) consistently improves performance of the baseline algorithm RegCB \citep{foster2018practical} which uses the squared loss instead of log-likelihood. It's worth noting that cost distribution learning has been shown to be effective in inverse-gap weighted (IGW) algorithms \citep{wang2023the}; however, our focus here is on optimistic algorithms such as \cb{} and RegCB. We now describe our efficient implementation with neural networks as function approximators via computing width with the log-likelihood loss.

\paragraph{Efficient Implementation by Computing Width.}
We group incoming contexts into batches $\Bcal_k\subset\Xcal$ to use GPU parallelism for neural networks. Let $\Dcal_{k-1}$ denote the history so far. Then, recall that \cb{} aims to compute optimistic actions $a_k=\argmin_a \min_{f\in\Fcal_{k-1}}\bar f(x_k,a)$ for each context $x_k\in\Bcal_k$, where $\Fcal_{k-1}$ is the subset of $\beta$-optimal functions w.r.t. the log-likelihood on the history $\Lcal_{\textsf{CB}}(f,\Dcal_{k-1})$, where $\beta$ is a hyperparameter.
We consider inducing optimism by subtracting the width of $\Fcal_{k-1}$, defined as
\begin{align*}
    w_k(x,a)
    =\max_{f,f'\in\Fcal}\braces{ \bar f(x,a)-\bar f'(x,a) }
    \,\,\text{s.t.}\,\,f,f'\in\Fcal_{k-1}.
\end{align*}
Then, given the MLE $g_k = \argmax_{g\in\Fcal}\Lcal_{\textsf{CB}}(g,\Dcal_{k-1})$ we can set $f_k:=(\bar g_k-w_k)$ which satisfies optimism, \ie, $f_k(x_k,a)\leq \bar C(x_k,a)$, for all $a$. Thus, the goal now is to compute $w_k(x_k,a)$ for each $x_k\in\Bcal_k$ and $a\in\Acal$.

We modify the width computation strategy of \citet{feng2021provably} to deal with the log-likelihood loss. In particular, given the current MLE $g_k$ parameterized by a neural net, we create a copy $g'$ and train $g'$ for a few steps of gradient ascent on the disagreement objective ($g_k$ is fixed):
\begin{align*}
&\sum_{a\in\Acal}\sum_{x_k\in\Bcal_k} \lambda (\bar g'(x_k,a) - \bar g_k(x_k,a))^2 / |\Bcal_k|
\\-&\sum_{(x,a)\in\Dcal_{k-1}} (\bar g' (x,a) - \bar g_k(x,a))^2 / |\Dcal_{k-1}|
\\-&\sum_{a \in \Acal}\sum_{x_k\in\Bcal_k} \lambda_1 (\bar g'(x_k,a) - \bar g_k(x_k,a)) / |\Bcal_k|
\end{align*}
where, the last term of the maximization objective is to avoid a zero gradient when $g_k=g'$. Due to memory constraints, we approximate the second term with a subset of the history. Then, we denote $\wh w_k(x,a) = |\bar g_k(x,a) - \bar g'(x,a)|$ and set the bonus to be the normalized width $\lambda_2 \cdot \frac{\wh w_k(s,a)}{\max_{a\in\Acal, x\in\Bcal_k} \wh w_k(x,a)}$. $\lambda,\lambda_1,\lambda_2$ are hyperparameters.

We note that an alternative poly-time algorithm is to binary search for a Lagrange multiplier as in RegCB \citep{foster2018practical}, which we also tried. However, the binary search approach requires an optimization oracle at every binary search depth, for every action, whereas disagreement computation only needs one optimization oracle per batch of contexts. Binary searching is thus much more computationally costly and we did not observe any improvement in performance to justify the increased computation. Hence, we use disagreement-based width computation for inducing optimism for all \cb{} and RegCB experiments.

\paragraph{CB Tasks.}
We now compare \cb{} and RegCB on the three real-world CB tasks: King County Housing \citep{OpenML2013}, Prudential Life Insurance \citep{prudential-life-insurance-assessment}, and CIFAR-100 \citep{cifar100}. The Housing and Prudential tasks are derived from risk prediction tasks, where a fixed max cost is incurred for over-predicting risk and a low cost is incurred for under-predicting risk \citep{farsang2022conditionally}. The CIFAR-100 task is derived from the image classification task, where $0$ cost is given for the correct label, $0.5$ cost is given for an almost correct label (\ie, correct superclass), and $1$ cost is given otherwise (for wrong superclass). All tasks were rolled out for $5000$ steps in batches of $32$ examples.

\paragraph{Function Approximators.}
We use neural networks for squared loss regression in RegCB and maximum likelihood estimation in \cb{}. For the King County Housing dataset and the Prudential Life Insurance dataset, we used $2$ hidden-layer MLPs, while for CIFAR-100, we used ResNet-18 \citep{he2016deep}. This is the same setup as in \citet[Appendix K]{wang2023the}.

\paragraph{Results.}
Table~\ref{tab:results} shows that cost distribution learning in \cb{} consistently improves the costs and regret compared to the baseline squared loss method RegCB.
Also, \cref{fig:cost-curve} shows that \cb{} converges to a smaller cost much faster than RegCB. This reinforces that our stronger theory for MLE-based distribution learning indeed translates to more effective algorithms than standard squared loss regression. We note that in the Housing and Prudential tasks, our costs are actually lower and better than the previously reported numbers by IGW algorithms \citep{wang2023the}. However, it is worth noting that optimistic algorithms based on width computation is still more computationally costly than IGW algorithms, and a carefully tuned IGW can likely perform just as well in practice.

\begin{table}
\begin{center}
\label{tab:results}
\centering
\tabcolsep=1pt
\begin{small}
\begin{tabular}{lcc}
\toprule
Algorithm: & RegCB & DistUCB (Ours) \\
\midrule
\multicolumn{3}{l}{King County Housing \citep{OpenML2013}} \\
\midrule
All episodes      &\hspace{0.5cm} .708 (.051) & \tb{.683} (.057) \\
Last 100 ep.      &\hspace{0.5cm} .676 (.038) & \tb{.640} (.037) \\
\midrule
\multicolumn{3}{l}{Prudential Life Insurance \citep{prudential-life-insurance-assessment}} \\
\midrule
All episodes      &\hspace{0.5cm} .287 (.058) & \tb{.248} (.061) \\
Last 100 ep.      &\hspace{0.5cm} .278 (.055) & \tb{.236} (.054) \\
\midrule
\multicolumn{3}{l}{CIFAR-100 \citep{cifar100}} \\
\midrule
All episodes      &\hspace{0.5cm} .890 (.053) & \tb{.862} (.058) \\
Last 100 ep.      &\hspace{0.5cm} .854 (.053) & \tb{.823} (.060) \\
\bottomrule
\end{tabular}
\vspace{-0.25cm}
\end{small}
\end{center}
\caption{Average cost over all episodes and last 100 episodes (lower is better). We report `mean (sem)' over $3$ seeds. }
\end{table}

\begin{figure}
    \centering
    \includegraphics[width=0.4\textwidth]{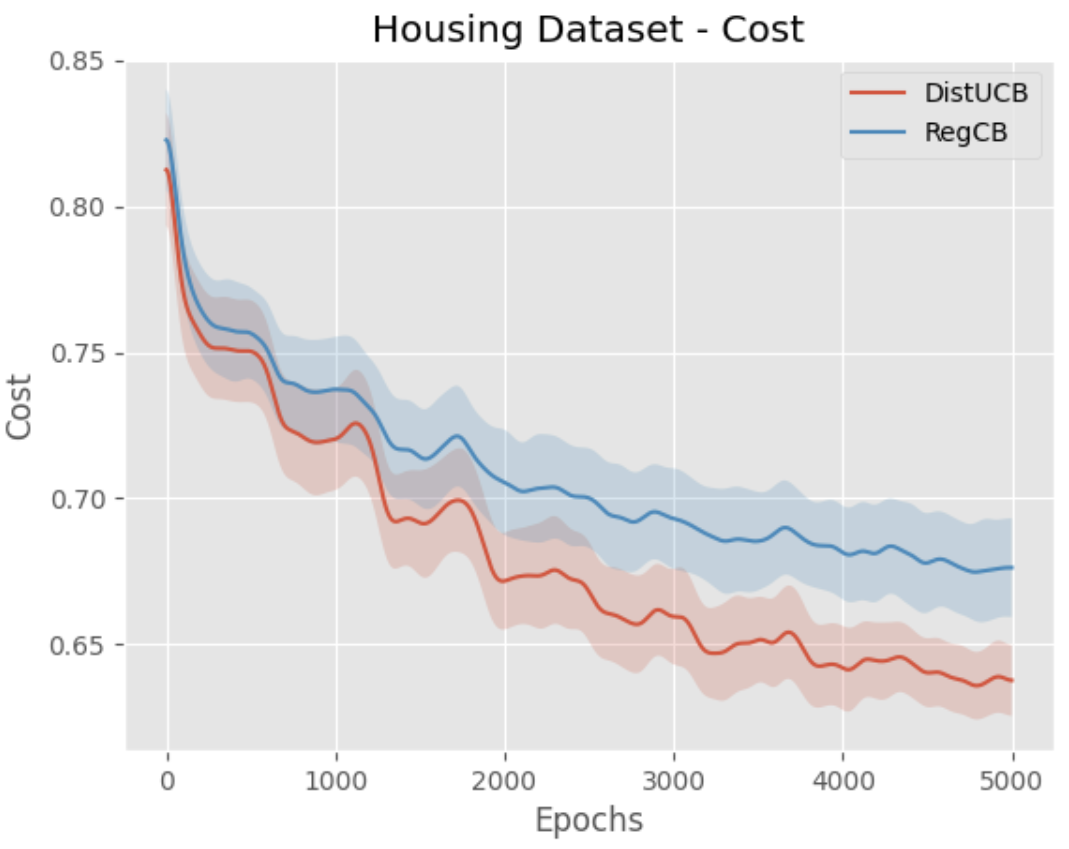}
    \vspace*{-3mm}
    \caption{Cost curves for the Housing task (lower is better).}
    \label{fig:cost-curve}
\end{figure}

\section{Conclusion}
We proved that MLE-based DistRL attains second-order bounds in both online and offline RL, significantly sharpening the previous results of \citet{wang2023the} and further showing the finite-sample statistical benefits of DistRL. In the CB case, we also proved a novel first and second-order gap-dependent bound and implemented the algorithm, showing it outperforms the previous squared loss method.
An interesting direction is to show whether DistRL can obtain even higher-order bounds than second-order.

\bibliography{paper}
\bibliographystyle{plainnat}

\newpage
\appendix
\onecolumn
\begin{center}\LARGE
\textbf{Appendices}
\end{center}

\section{Notations}

{\renewcommand{\arraystretch}{1.3}%
\begin{table}[h!]
    \centering
      \caption{List of Notations} \vspace{0.3cm}
    \begin{tabular}{l|l}
    $\Scal, \Acal, A$ & State and action spaces, and $A = |\Acal|$. \\
    $\Delta(S)$ & The set of distributions supported by set $S$. \\
    $\bar d$ & The expectation of any real-valued distribution $d$, \ie, $\bar d = \int y d(y)\diff\lambda(y)$. \\
    $\Var(d)$ & The variance of any real-valued distribution $d$, \ie, $\Var(d) = \int(y-\bar d)^2d(y)\diff\lambda(y)$. \\
    $[N]$ & $\braces{1,2,\dots,N}$ for any $N\in\NN$. \\
    $Z^\pi_h(x,a)$ & Distribution of $\sum_{t=h}^H c_t$ given $x_h=x,a_h=a$ rolling in from $\pi$. \\
    $Q^\pi_h(x,a),V^\pi_h(x)$ & $Q^\pi_h(x,a)=\bar Z^\pi_h(x,a)$ and $V^\pi_h=\Eb[a\sim\pi(x)]{Q^\pi_h(x,a)}$. \\
    $Z^\star_h,Q^\star_h,V^\star_h$ & $Z^\pi_h,Q^\pi_h,V^\pi_h$ with $\pi=\pi^\star$, the optimal policy. \\
    $\Tcal_h^\pi, \Tcal_h^\star$ & The Bellman operators that act on functions. \\
    $\Tcal_h^{\pi,D},\Tcal_h^{\star,D}$ & The distributional Bellman operators that act on conditional distributions. \\
    $V^\pi,Z^\pi,V^\star,Z^\star$ & $V^\pi = V^\pi_1(x_1)$, $Z^\pi=Z^\pi_1(x_1)$. $V^\star,Z^\star$ are defined similarly with $\pi^\star$. \\
    $d^{\pi}_h(x,a)$ & The probability of $\pi$ visiting $(x,a)$ at time $h$. \\
    $C^{\wt\pi}$ & Coverage coefficient $\max_h\nm{\nicefrac{\diff d^{\wt\pi}_h}{\diff\nu_h}}_\infty$. \\
    $D_\triangle(f\Mid g)$ & Triangular discrimination between $f,g$. \\
    $H(f\Mid g)$ & Hellinger distance between $f,g$.
    \end{tabular}
    \label{tab:notation}
\end{table}
}

\subsection{Statistical Distances}
Let $f,g$ be distributions over $\Ycal$. Then,
\begin{align*}
    &D_\triangle(f\Mid g) = \sum_y \frac{\prns{f(y)-g(y)}^2}{f(y)+g(y)},
    \\&H^2(f\Mid g) = \frac{1}{2}\sum_y \prns{ \sqrt{f(y)}-\sqrt{g(y)} }^2.
\end{align*}

\begin{lemma}\label{lemma:hellinger-triangle-equiv}
For any distributions $f,g$, we have $2H^2(f\Mid g)\leq D_\triangle(f\Mid g)\leq 4H^2(f\Mid g)$.
\end{lemma}
\begin{proof}
Recall that $D_\triangle(f\Mid g)=\int_y \prns{\frac{f(y)-g(y)}{\sqrt{f(y)+g(y)}}}^2$.
Apply $\frac{1}{\sqrt{f(y)}+\sqrt{g(y)}}\leq \frac{1}{\sqrt{f(y)+g(y)}}\leq \frac{\sqrt{2}}{\sqrt{f(y)}+\sqrt{g(y)}}$.
\end{proof}

\newpage
\section{Proofs for CB Lemmas}

\varKeyIneqOne*
\begin{proof}
For any constant $c$ and random variable $X$, recall that $\EE(X-c)^2=\Var(X)+(\EE X - c)^2$.
Thus,
\begin{align*}
    \bar f-\bar g
    &= \sum_z (z-c)(f(z)-g(z))
    \\&\leq \sqrt{ \sum_z (z-c)^2 (f(z)+g(z)) } \cdot\sqrt{ \sum_z \frac{(f(z)-g(z))^2}{f(z)+g(z)} } \tag{Cauchy-Schwartz}
    \\&= \sqrt{ \Var(f) + \Var(g) + \prns*{ \bar f - c }^2 + \prns*{ \bar g - c }^2 } \cdot\sqrt{ D_\triangle(f\Mid g) }.
\end{align*}
To minimize the bound, set $c = \frac{\bar f+\bar g}{2}$ to get,
\begin{align*}
    \abs{\bar f-\bar g}
    &\leq \sqrt{ \Var(f) + \Var(g) + \prns*{ \bar f - \bar g }^2 /2 }\cdot\sqrt{ D_\triangle(f\Mid g) }
    \\&\leq \sqrt{ (\Var(f) + \Var(g))D_\triangle(f\Mid g) } + \abs{\bar f-\bar g}\sqrt{D_\triangle(f\Mid g)/2}.
\end{align*}
Rearranging and using the fact that $D_\triangle(f\Mid g)<2$,
\begin{equation*}
    \abs{\bar f - \bar g}\leq \frac{1}{1-\sqrt{D_{\triangle}(f\Mid g)/2}}\sqrt{(\Var(f) + \Var(g)) D_\triangle(f\Mid g) }.
\end{equation*}
Finally, use the facts that $\frac{1}{1-\eps}\leq2$ for $\eps\in[0,\frac12]$ and $\sqrt{D_\triangle(f\Mid g)/2}\leq\frac12$ by the premise.
\end{proof}

\begin{lemma}
For any $f,g\in\Delta([0,1])$, we have
\begin{equation}
    \abs{\Var(f)-\Var(g)}\leq 4\sqrt{ (\Var(f) + D_\triangle(f\Mid g))D_\triangle(f\Mid g)} \label{eq:diff-of-var}
\end{equation}
\end{lemma}
\begin{proof}
Recall that $\Var(f) = \frac12\EE_{z,z'\sim f\otimes f}[(z-z')^2]$.
So if $f'$ is the distribution of $\frac12(z-z')^2$ where $z,z'\sim f$, then $\Var(f) = \bar f'$.
Since $(z-z')^2\in[0,1]$, we can use Eq.($\Delta_2$) of \citep{wang2023the} to get $\abs{\bar f-\bar g}\leq \sqrt{(4\bar f+D_\triangle(f\Mid g))D_\triangle(f\Mid g)}$.
Thus, $\abs{\Var f-\Var g}\leq\sqrt{ (4\Var_f+D_\triangle(f'\Mid g'))D_\triangle(f'\Mid g') }$.

Now it suffices to bound $D_\triangle(f'\Mid g')$ by $4D_\triangle(f\Mid g)$, which we do by data processing inequality and tensorization of Hellinger.
In particular, the tensorization of $H^2$ is given by
$H^2(f\otimes f\Mid g\otimes g)=2-2(1-H^2(f\Mid g)/2)^2$ \citep[Eqn. 7.26]{infotheorybook} and
using $1-(1-x/2)^2\leq x$ implies that $H^2(f\otimes f\Mid g\otimes g)\leq 2H^2(f\Mid g)$.
Thus,
\begin{align*}
    D_\triangle(f'\Mid g')
    &\leq D_\triangle(f\otimes f\Mid g\otimes g) \tag{data processing ineq.}
    \\&\leq 4H^2(f\otimes f\Mid g\otimes g) \tag{$D_\triangle\leq 4H^2$}
    \\&\leq 8H^2(f\Mid g) \tag{tensorization of $H^2$}
    \\&\leq 4D_\triangle(f\Mid g). \tag{$2H^2\leq D_\triangle$}
\end{align*}
\end{proof}

\varKeyIneqTwo*
\begin{proof}
If $D_\triangle(f\Mid g)>\frac12$, then we trivially have $\abs{ \bar f-\bar g } \leq 1 \leq 2D_\triangle(f\Mid g)$ since $\bar f,\bar g\in[0,1]$. Thus, we can assume $D_\triangle(f\Mid g)\leq\frac12$.
Starting from \cref{eq:var-key-ineq1}, we can bound the sum of two variances as follows,
\begin{align*}
\Var(f) + \Var(g)
&= 2\Var(f) + \Var(g) - \Var(f)
\\&\leq 2\Var(f) + 4\sqrt{(\Var(f) + D_\triangle(f\Mid g))D_\triangle(f\Mid g)} \tag{\cref{eq:diff-of-var}}
\\&\leq 2\Var(f) + 4\sqrt{\Var(f) D_\triangle(f\Mid g)} + 4D_\triangle(f\Mid g)
\\&\leq 4\Var(f) + 6D_\triangle(f\Mid g). \tag{AM-GM}
\end{align*}
Hence, we have
\begin{align*}
\abs{\bar f-\bar g}
&\leq 2\sqrt{ (\Var(f)+\Var(g))D_\triangle(f\Mid g) }
\\&\leq 2\sqrt{(4\Var(f) + 6D_\triangle(f\Mid g))D_\triangle(f\Mid g)} \tag{above inequality}
\\&\leq 4\sqrt{\Var(f) D_\triangle(f\Mid g)} + 5D_\triangle(f\Mid g).
\end{align*}
This finishes the proof.
\end{proof}

\section{Proof for Gap-dependent Bounds for CB}\label{sec:gap-dependent-proof}

Define $\dimCB(\eps)=\eluDim(\braces{(x,a)\mapsto D_\triangle(f(x,a)\Mid C(x,a)): f\in\Fcal},\eps)$ is the $\ell_1$-eluder dimension at threshold $\eps$ \citep{liu2022partially}.

\CBGapDependentRegret*
\begin{proof}[Proof of \cref{thm:gap-dependent-regret}]
Define $\delta_k(x,a):=D_\triangle(f_k(x,a)\Mid C(x,a))$ and $\Delta=\sum_k\delta_k(x_k,a_k)$, the same notation as in \cref{sec:proof-of-second-order-cb}.
We partition episodes into burn-in and stable episodes, where stable episodes are those that satisfy: $\delta_k(x_k,a_k)\leq \Var(C(x_k,a_k))$.
Let $\Ecal$ denote the set of stable episodes and $\neg\Ecal$ are the burn-in episodes.

\paragraph{Step 1: burn-in episodes have $\Ocal(\Delta)$ regret.}
\begin{align*}
    \sum_{k\in\Ecal_1\cap\Ecal_2^C} \bar C(x_k,a_k)-\min_a\bar C(x_k,a)
    &\leq \sum_{k\in\Ecal_1\cap\Ecal_2^C} \bar C(x_k,a_k)-\bar f_k(x_k,a_k) \tag{optimism}
    \\&\leq \sum_{k\in\Ecal_1\cap\Ecal_2^C} 4\sqrt{\Var(C(x_k,a_k))\delta_k(x_k,a_k)}+5\delta_k(x_k,a_k) \tag{\cref{eq:var-key-ineq2}}
    \\&\leq \sum_{k\in\Ecal_1\cap\Ecal_2^C} 4\delta_k(x_k,a_k)+5\delta_k(x_k,a_k) \tag{$\neg\Ecal$}
    \\&\leq \sum_{k=1}^K 9 \delta_k(x_k,a_k) = 9\Delta.
\end{align*}
This implies that $\sum_{k\not\in\Ecal}\bar C(x_k,a_k)-\min_a\bar C(x_k,a) \leq 9\Delta$.

\paragraph{Step 2: stable episodes have gap-dependent regret.}
We now argue those episodes in $\Ecal$ have large gap.
For each $k$, optimism implies that $\bar f_k(x_k,a_k)\leq \min_a\bar C(x_k,a) = \bar C(x_k,a_k) -\gap(x_k,a_k)$.
This implies that $\gap(x_k,a_k)\leq \bar C(x_k,a_k) - \bar f_k(x_k,a_k)$.
By $\Ecal$, we have $4\sqrt{\Var(C(x_k,a_k))\delta_k(x_k,a_k)}+5\delta_k(x_k,a_k)\leq 9\sqrt{\Var(C(x_k,a_k))\delta_k(x_k,a_k)}$, and hence the previous display implies
\begin{equation*}
    \bar C(x_k,a_k) - \bar f_k(x_k,a_k)\leq 9\sqrt{\Var(C(x_k,a_k))\delta_k(x_k,a_k)}.
\end{equation*}
If this is zero, then the regret for the episode is zero. If this is non-zero, we have $\gap(x_k,a_k)>0$ and $\Var(C(x_k,a_k))>0$, which implies that
\begin{equation*}
    9\sqrt{\delta_k(x_k,a_k)}\geq \frac{\gap(x_k,a_k)}{\sqrt{\Var(C(x_k,a_k))}}\geq \vargap.
\end{equation*}

Now we will invoke the standard peeling technique (\cref{lem:peeling}) on $9\sqrt{\delta_k(x_k,a_k)}$.
For any $\zeta>0$, we have
\begin{equation}
    \sum_{k=1}^K \I{ \delta_k(x_k,a_k)\geq\zeta }\leq 4\dimCB(\zeta)\beta\log(K)\zeta^{-1}, \label{eq:eluder-zeta}
\end{equation}
because $\I{\delta_k(x_k,a_k)\geq\zeta}\leq \zeta^{-1}\delta_k(x_k,a_k)$ and the summation of $\delta_k(x_k,a_k)$ is bounded by the eluder dimension with log factors \citep[Theorem 5.3]{wang2023the}.
This indeed satisfies the assumption of \cref{lem:peeling} with $C=4\dimCB(\vargap^2)\beta\log(K)$.
Thus, we can bound the stable episode regret as follows:
\begin{align*}
    &\sum_{k\in\Ecal_1\cap\Ecal_2} \bar C(x_k,a_k)-\min_a\bar C(x_k,a)
    \\&\leq\sum_{k\in\Ecal_1\cap\Ecal_2} 9\sqrt{\Var(C(x_k,a_k))\delta_k(x_k,a_k)} \tag{same steps as before and $\Ecal$}
    \\&\leq\sum_{k\in\Ecal_1\cap\Ecal_2} 9\sqrt{\delta_k(x_k,a_k)} \tag{$C(x_k,a_k)\in[0,1]$}
    \\&\leq 18\cdot 16\dimCB(\vargap^2)\beta\log(K)\vargap^{-1}. \tag{$9\sqrt{\delta_k(x_k,a_k)}\leq 18$ and \cref{lem:peeling}}
\end{align*}
In the last inequality, note that we invoke \cref{lem:peeling} directly on $\sqrt{\delta_k}$.
Thus, we have shown the $\vargap$-dependent regret:
\begin{align*}
    \op{Regret}_{\textsf{CB}}(K)\leq 11\cdot 4\dimCB(K^{-1})\beta\log(K) + 288\frac{\dimCB(\vargap^2)\beta\log(K)}{\vargap}.
\end{align*}
Following the same steps, and using \cref{lem:Lstar-self-bounding}, we can prove the same result for $\CStarGap$.
Therefore, we have shown that
\begin{align*}
    \op{Regret}_{\normalfont\textsf{CB}}(K)\leq\wt\Ocal\prns{\dimCB + \min\braces{\frac{\dimCB(\vargap^2)}{\vargap}, \frac{\dimCB(\CStarGap^2)}{\CStarGap} } }.
\end{align*}

Finally, notice that if $\vargap\geq\frac{1}{\sqrt{K}}$, $\dimCB(\vargap^2)\leq \dimCB(1/K) = \dimCB$ by monotonicity of the eluder dimension. If $\vargap<\frac{1}{\sqrt{K}}$ then $1/\vargap \geq \sqrt{K}$ anyways, and so this small-gap regime results in a $\Ocal(\sqrt{K})$ bound; in this case, we already have a better second-order bound in \cref{thm:second-order-cb}.
This finishes the proof for \cref{thm:gap-dependent-regret}.
\end{proof}

\begin{lemma}\label{lem:Lstar-self-bounding}
For each episode $k$, we have
\begin{equation*}
    \bar C(x_k,a_k)-\min_a\bar C(x_k,a) \leq 3\sqrt{ \min_a\bar C(x_k,a) \cdot \delta_k(x_k,a_k) } + 6\delta_k(x_k,a_k).
\end{equation*}
\end{lemma}
\begin{proof}
By optimism and \citet[Equation $\Delta_2$]{wang2023the}, we have
\begin{equation*}
    \bar C(x_k,a_k)-\min_a\bar C(x_k,a)\leq \bar C(x_k,a_k)-\bar f_k(x_k,a_k)\leq 2\sqrt{ \bar C(x_k,a_k)\delta_k(x_k,a_k) } + \delta_k(x_k,a_k).
\end{equation*}
Using AM-GM, this can be further bounded by $\frac12\bar C(x_k,a_k) + 3\delta_k(x_k,a_k)$. Rearranging, this implies $\bar C(x_k,a_k)\leq 2\min_a\bar C(x_k,a) + 6\delta_k(x_k,a_k)$. Therefore, plugging this back into above,
\begin{align*}
    \bar C(x_k,a_k)-\min_a\bar C(x_k,a)
    &\leq 2\sqrt{ (2\min_a\bar C(x_k,a)+6\delta_k(x_k,a_k))\delta_k(x_k,a_k) } + \delta_k(x_k,a_k)
    \\&\leq 3\sqrt{\min_a \bar C_k(x_k,a)\cdot\delta_k(x_k,a_k)} + 6\delta_k(x_k,a_k).
\end{align*}
\end{proof}

\begin{lemma}[Peeling Lemma]\label{lem:peeling}
Suppose $g_1,g_2,\dots,g_K:\Zcal\to[0,1]$ and $z_1,z_2,\dots,z_K\in\Zcal$ satisfy $g_k(z_k)\geq\gap$ for all $k$.
Moreover, suppose there exists $C>0$ such that for any $\zeta\geq\gap$, we have $\sum_k \I{g_k(z_k)\geq\zeta}\leq C\zeta^{-2}$.
Then,
\begin{equation*}
    \sum_{k=1}^Kg_k(z_k)\leq 4C\gap^{-1}.
\end{equation*}
\end{lemma}
\begin{proof}
Divide $[\gap,1]$ into $N=\ceil{\log(1/\gap)}$ intervals, where the $i\in[N]$-th interval is $[2^{i-1}\gap, 2^i\gap)$. Then, we bound the sum via a standard peeling argument: note that $g_k(z_k)\I{ g_k(z_k)\in[2^{i-1}\gap, 2^i\gap) }\leq 2^i\gap \I{ g_k(z_k)\geq 2^{i-1}\gap }$. Therefore,
\begin{align*}
    \sum_kg_k(z_k)
    &= \sum_k \sum_{i=1}^N g_k(z_k)\I{ g_k(z_k)\in[2^{i-1}\gap, 2^i\gap) }
    \\&\leq \sum_k \sum_{i=1}^N 2^i\gap \I{ g_k(z_k)\geq 2^{i-1}\gap }
    \\&\leq \sum_{i=1}^N 2^i\gap \cdot C 2^{-2i+2}\gap^{-2} \tag{premise}
    \\&= 4C\gap^{-1} \sum_{i=1}^N 2^{-i} \leq 4C\gap^{-1}.
\end{align*}
\end{proof}

\section{RL Lemmas}

\begin{lemma}[Performance Difference]\label{lem:performance-difference}
For any $f:(\Xcal\times\Acal\to\RR)^H$, policy $\pi$ and $x_1$, we have
\begin{equation*}
    V^\pi(x_1)-f_1(x_1,\pi(x_1))=\sum_{h=1}^H\EE_{\pi,x_1}\bracks{ \prns{\Tcal^\pi_h f_{h+1}-f_h}(x_h,a_h) }.
\end{equation*}
\end{lemma}
\begin{proof}
See \citet[Lemma H.2]{wang2023the}.
\end{proof}

\begin{restatable}[Second-order implies Small-loss]{theorem}{secondOrderImpliesSmallLoss}\label{thm:second-order-implies-small-loss}
For online RL, suppose we have a second-order bound: $\sum_{k=1}^K V^{\pi^k}(x_{1,k})-V^\star(x_{1,k})\leq \sqrt{c \sum_{k=1}^K\Var(Z^{\pi^k}(x_{1,k}))} + c$, for some $c\in\RR_+$. Then, we also have a small-loss (first-order) bound: $\sum_{k=1}^K V^{\pi^k}(x_{1,k})-V^\star(x_{1,k})\leq \sqrt{2c \sum_{k=1}^K V^\star(x_{1,k})} + 3c$.

For offline RL, suppose we have a second-order bound w.r.t. comparator policy $\pi_{\normalfont\text{comp}}$: $V^{\wh\pi}-V^{\pi_{\normalfont\text{comp}}}\leq \sqrt{ \frac{c'\Var(Z(\pi_{\normalfont\text{comp}}))}{N} } + \frac{c'}{N}$. Then, we also have a small-loss (first-order) bound: $V^{\wh\pi}-V^{\pi_{\normalfont\text{comp}}}\leq \sqrt{ \frac{c'V^{\pi_{\normalfont\text{comp}}}}{N} } + \frac{c'}{N}$.
\end{restatable}
\begin{proof}
The offline RL claim follows from $\Var(Z(\pi_{\normalfont\text{comp}}))\leq V^{\pi_{\normalfont\text{comp}}}$ because returns are bounded between $[0,1]$ and variance is bounded by second moment, which is bounded by first moment.
So, we will focus on the online RL claim for the remainder of the proof.
\begin{align}
    \sum_{k=1}^K V^{\pi^k}(x_{1,k})-V^\star(x_{1,k})
    &\leq \sqrt{c \sum_{k=1}^K\Var(Z^{\pi^k}(x_{1,k}))} + c \tag{premise}
    \\&\leq \sqrt{c \sum_{k=1}^K V^{\pi^k}(x_{1,k})} + c \label{eq:second-order-implies-small-loss-implicit}
    \\&\leq \frac12 c + \frac12 \sum_{k=1}^KV^{\pi^k}(x_{1,k}) + c, \tag{AM-GM}
\end{align}
which implies
\begin{equation*}
    \sum_{k=1}^K V^{\pi^k}(x_{1,k})\leq 2\sum_{k=1}^K V^\star(x_{1,k}) + 3c.
\end{equation*}
Plugging this back into \cref{eq:second-order-implies-small-loss-implicit} gives
\begin{align*}
    \sum_{k=1}^K V^{\pi^k}(x_{1,k})-V^\star(x_{1,k})\leq \sqrt{ 2c\sum_{k=1}^KV^\star(x_{1,k})+3c^2 } + c,
\end{align*}
which finishes the proof.
\end{proof}

\subsection{Variance Change of Measure}
\mainPaperChangeOfVarianceLemma*
\begin{proof}
Apply law of total variance to the variance term of \cref{thm:main-variance-change-of-measure}, \ie,
\begin{align*}
    \Var(Z^\pi_1(x_1))
    &=\EE_{\pi,x_1}[\Var(f_h(x_h,a_h)\mid x_h,a_h,x_1)\mid x_1] + \Var_{\pi,x_1}\prns{ \EE\bracks{ f_h(x_h,a_h)\mid x_h,a_h,x_1 }\mid x_1 }
    \\&\geq \EE_{\pi,x_1}[\Var(f_h(x_h,a_h)\mid x_h,a_h,x_1)\mid x_1].
\end{align*}
\end{proof}
\begin{theorem}\label{thm:main-variance-change-of-measure}
Fix any $f:\Xcal\times\Acal\to\Delta([0,1])$ and any policy $\pi$.
Define $\delta_h(x,a):=D_\triangle(f_h(x,a)\Mid\Tcal^{\pi,D}_h f_{h+1}(x,a))$ and $\Delta_h(x_h,a_h):=\sum_{t=h}^H\EE_{\pi,x_h,a_h}\bracks{\delta_t(x_t,a_t)}$.
Then, for all $h\in[H],x_h,a_h$, we have
\begin{equation}
    \Var(f_h(x_h,a_h))\leq 2e\Var(Z^\pi_h(x_h,a_h))+12H(H-h+1)\Delta_h(x_h,a_h).
\end{equation}
Therefore, for any $x_1$,
\begin{equation}
    \EE_{\pi,x_1}\bracks{\Var(f_h(x_h,a_h))}\leq 2e\Var(Z^\pi_1(x_1)) + 12H^2\EE_{\pi,x_1}\bracks{\Delta_h(x_h,a_h)}.
\end{equation}
\end{theorem}
\begin{proof}
The main technical lemma is \cref{lem:var-change-of-measure-main}, which is proven with induction.
Given this lemma, use the fact that $(1+H^{-1})^H\leq e$ to get
\begin{equation*}
    \Var(f_h(x_h,a_h))
    \leq \sum_{t=h}^H e\prns{\Eb[\pi,x_h,a_h]{2\Var_{c_t,x_{t+1}}(c_t+V^\pi_{t+1}(x_{t+1})\mid x_t,a_t) + 12H\Delta_t(x_t,a_t)}}.
\end{equation*}
Recall that $\Var(Z^\pi_h(x_h,a_h)) = \sum_{t=h}^H\EE_{\pi,x_h,a_h}\bracks{ \Var_{c_t,x_{t+1}}(c_t+V^\pi_{t+1}(x_{t+1})\mid x_t,a_t) }$, by the law of total variance.
Also for any $t\geq h$, we have $\EE_{\pi,x_h,a_h}\Delta_t(x_t,a_t)\leq \Delta_h(x_h,a_h)$. Thus,
\begin{equation*}
    \Var(f_h(x_h,a_h))
    \leq 2e \Var(Z^\pi_h(x_h,a_h)) + 12H(H-h+1)\Delta_h(x_h,a_h),
\end{equation*}
which proves the claim.
\end{proof}

\begin{lemma}\label{lem:var-change-of-measure-main}
For all $h\in[H],x_h,a_h$, we have
\begin{align}
    \Var(f_h(x_h,a_h))
    &\leq \sum_{t=h}^H(1+H^{-1})^{t-h+1}\prns{\Eb[\pi,x_h,a_h]{2\Var_{c_t,x_{t+1}}(c_t+V^\pi_{t+1}(x_{t+1})\mid x_t,a_t) + 12H\Delta_t(x_t,a_t)}}. \label{eq:var-change-of-measure-IH}
\end{align}
\end{lemma}
\begin{proof}
First observe that
\begin{equation}
    \Var(f_h(x_h,a_h))\leq \prns*{1+H^{-1}}\Var(\Tcal^{\pi,D}_h f_{h+1}(x_h,a_h)) + 12H\delta_h(x_h,a_h), \label{eq:var-change-of-measure-first-step}
\end{equation}
because by \cref{eq:diff-of-var} and AM-GM, we have
\begin{align*}
    \Var(f_h(x_h,a_h))-\Var(\Tcal^{\pi,D}_h f_{h+1}(x_h,a_h))
    &\leq 4\sqrt{(\Var(\Tcal^{\pi,D}_h f_{h+1}(x_h,a_h))+\delta_h(x_h,a_h))\delta_h(x_h,a_h)}
    \\&\leq 4\sqrt{\Var(\Tcal^{\pi,D}_h f_{h+1}(x_h,a_h))\delta_h(x_h,a_h)} + 4\delta_h(x_h,a_h)
    \\&\leq H^{-1}\Var(\Tcal^{\pi,D}_h f_{h+1}(x_h,a_h)) + 8H\delta_h(x_h,a_h)+4\delta_h(x_h,a_h).
\end{align*}

We now proceed to show \cref{eq:var-change-of-measure-IH} by induction.
The base case $h=H$ is true since $\Var(\Tcal^{\pi,D}_H f_{H+1}(x_H,a_H))=\Var(C_H(x_H,a_H))=\Var(c_H+V^\pi_{H+1}(x_{H+1})\mid x_H,a_H)$.

We now prove the induction step: suppose the \cref{eq:var-change-of-measure-IH} is true for $h+1$; we want to show the $h$ case is true.
By the law of total conditional variance, we have that $\Var(\Tcal^{\pi,D}_h f_{h+1}(x_h,a_h))$ is equal to:
\begin{align*}
    &\EE\bracks{ \Var(c_h + f_{h+1}(x_{h+1},\pi(x_{h+1}))\mid x_{h+1},c_h,x_h,a_h) \mid x_h,a_h}
    + \Var\prns{ \EE\bracks{c_h+f_{h+1}(x_{h+1},\pi(x_{h+1}))\mid x_{h+1},c_h,x_h,a_h }\mid x_h,a_h }
    \\&=\,\EE\bracks{ \Var(f_{h+1}(x_{h+1},\pi(x_{h+1}))\mid x_{h+1}) \mid x_h,a_h}
    + \Var_{c_h,x_{h+1}\sim C_h,P_h(x_h,a_h)}\prns{ c_h + \bar f_{h+1}(x_{h+1},\pi(x_{h+1}))}.
\end{align*}
The first term is controlled by the induction hypothesis.
The second term is handled by \cref{lem:var-change-of-measure-second-term}. Therefore,
\begin{align*}
    &\Var(\Tcal^{\pi,D}_h f_{h+1}(x_h,a_h))
    \\&\leq \EE_{\pi,x_h,a_h}\sum_{t=h+1}^H(1+H^{-1})^{t-h} \prns{2\EE_{\pi,x_{h+1},a_{h+1}}\bracks{ \Var_{c_t,x_{t+1}}\prns{ c_t+V^\pi_{t+1}(x_{t+1}) \mid x_t,a_t } + 12H\Delta_{t}(x_{t},a_{t})} }
    \\&+ 2\Var_{c_h,x_{h+1}\sim C_h,P_h(x_h,a_h)}\prns{ c_h+V^\pi_{h+1}(x_{h+1}) } + 4H\EE_{\pi,x_h,a_h}\Delta_{h+1}(x_{h+1},a_{h+1}).
\end{align*}

Thus, by \cref{eq:var-change-of-measure-first-step}, we have
\begin{align*}
    &\Var(f_h(x_h,a_h))
    \\&\leq \EE_{\pi,x_h,a_h}\sum_{t=h+1}^H(1+H^{-1})^{t-h+1} \prns{2\EE_{\pi,x_{h+1},a_{h+1}}\bracks{ \Var_{c_t,x_{t+1}}\prns{ c_t+V^\pi_{t+1}(x_{t+1}) \mid x_t,a_t } + 12H\Delta_{t}(x_{t},a_{t}) }}
    \\&+ (1+H^{-1})\prns{2\Var_{c_h,x_{h+1}\sim C_h,P_h(x_h,a_h)}\prns{ c_h+V^\pi_{h+1}(x_{h+1}) } + 4H\EE_{\pi,x_h,a_h}\Delta_{h+1}(x_{h+1},a_{h+1})}
    + 12H\delta_h(x_h,a_h)
    \\&\leq \sum_{t=h+1}^H(1+H^{-1})^{t-h+1} \prns{2\EE_{\pi,x_h,a_h}\bracks{ \Var_{c_t,x_{t+1}}\prns{ c_t+V^\pi_{t+1}(x_{t+1}) \mid x_t,a_t } + 12H\Delta_{t}(x_{t},a_{t}) }}
    \\&+ (1+H^{-1})\prns{2\Var_{c_h,x_{h+1}\sim C_h,P_h(x_h,a_h)}\prns{ c_h+V^\pi_{h+1}(x_{h+1}) } + 12H\Delta_{h}(x_{h},a_{h})}
    \\&=\sum_{t=h}^H(1+H^{-1})^{t-h+1} \prns{2\EE_{\pi,x_h,a_h}\bracks{ \Var_{c_t,x_{t+1}}\prns{ c_t+V^\pi_{t+1}(x_{t+1}) \mid x_t,a_t } + 12H\Delta_{t}(x_{t},a_{t}) }},
\end{align*}
which finishes the induction.
\end{proof}

\begin{lemma}\label{lem:var-change-of-measure-second-term}
\begin{align*}
    &\Var_{c_h,x_{h+1}\sim C_h,P_h(x_h,a_h)}\prns{ c_h + \bar f_{h+1}(x_{h+1},\pi(x_{h+1})) }
    \\&\leq 2\Var_{c_h,x_{h+1}\sim C_h,P_h(x_h,a_h)}\prns{ c_h + V^\pi_{h+1}(x_{h+1}) } + 4(H-h)\EE_{\pi,x_h,a_h}\Delta_{h+1}(x_{h+1},a_{h+1}).
\end{align*}
\end{lemma}
\begin{proof}
Recall that $\Var(X+Y)\leq 2\Var(X)+2\Var(Y)$ and hence,
\begin{align*}
&\Var_{c_h,x_{h+1}\sim C_h,P_h(x_h,a_h)}\prns{ c_h + \bar f_{h+1}(x_{h+1},\pi(x_{h+1})) }
\\&\leq 2\Var_{c_h,x_{h+1}\sim C_h,P_h(x_h,a_h)}\prns{ c_h + V^\pi_{h+1}(x_{h+1}) } + 2\Var_{x_{h+1}\sim P_h(x_h,a_h)}\prns{ \bar f_{h+1}(x_{h+1},\pi(x_{h+1}))-V^\pi_{h+1}(x_{h+1}) }
\end{align*}
For the second term, we first bound the envelope of $\bar f_{h+1}(x_{h+1},\pi(x_{h+1}))-V^\pi_{h+1}(x_{h+1})$ as follows:
\begin{align*}
    \abs{\bar f_{h+1}(x_{h+1},\pi(x_{h+1}))-V^\pi_{h+1}(x_{h+1})}
    &\leq \sum_{t=h+1}^H \Eb[\pi,x_{h+1}]{ \abs{ \bar f_t(x_t,a_t)-\Tcal^\pi_t\bar f_{t+1}(x_t,a_t) } } \tag{PDL}
    \\&\leq \sum_{t=h+1}^H \Eb[\pi,x_{h+1}]{ \sqrt{ 2\delta_t(x_t,a_t) } } \tag{Eq.($\Delta_1$) of \citet{wang2023the}}
\end{align*}
This enables us to bound the variance,
\begin{align*}
    &\Var_{x_{h+1}\sim P_h(x_h,a_h)}\prns{ \bar f_{h+1}(x_{h+1},\pi(x_{h+1}))-V^\pi_{h+1}(x_{h+1}) }
    \\&\leq \EE_{x_{h+1}\sim P_h(x_h,a_h)}\bracks{ \prns{ \bar f_{h+1}(x_{h+1},\pi(x_{h+1}))-V^\pi_{h+1}(x_{h+1}) }^2 }
    \\&\leq \EE_{x_{h+1}\sim P_h(x_h,a_h)}\bracks{ \prns{ \sum_{t=h+1}^H \EE_{\pi,x_{h+1}}\bracks{ \sqrt{2\delta_t(x_t,a_t)} } }^2 }
    \\&\leq (H-h)\EE_{x_{h+1}\sim P_h(x_h,a_h)}\bracks{ \sum_{t=h+1}^H \prns{ \EE_{\pi,x_{h+1}}\bracks{ \sqrt{2\delta_t(x_t,a_t)} } }^2 }
    \\&\leq (H-h)\EE_{x_{h+1}\sim P_h(x_h,a_h)}\bracks{ \sum_{t=h+1}^H \EE_{\pi,x_{h+1}}\bracks{ 2\delta_t(x_t,a_t) } },
\end{align*}
as desired.
\end{proof}

\section{Proofs for Online RL}

\secondOrderRL*
\begin{proof}[Proof of \cref{thm:second-order-rl}]
As noted by \citep[Proof of Theorem 5.5]{wang2023the}, the confidence set construction guarantees two facts w.p. $1-\delta$:
\textbf{for all} $k\in[K]$,
\begin{enumerate}[label=(\roman*)]
    \item Optimism: $\min_a\bar f_1^{(k)}(x_{1,k},a)\leq V^\star(x_{1,k})$ (since $Z^\pi(x_{1,k})\in\Fcal_k$); and
    \item Small-generalization error: for all $h$, we have
    \begin{enumerate}[leftmargin=3cm]
        \item[If \textsc{UAE}=\textsc{False}.] $\sum_{i<k}\Eb[\pi^i]{\delta_{h,k}(s_h,a_h)}\leq c\beta$;
        \item[If \textsc{UAE}=\textsc{True}.] $\sum_{i<k}\Eb[\pi^i]{\EE_{a'\sim\op{unif}(\Acal)}[\delta_{h,k}(s_h,a_h)]}\leq c\beta$,
    \end{enumerate}
    for some universal constant $c$.
\end{enumerate}
Let $\delta_{h,k}(x,a):=D_\triangle(f^{(k)}_h(x,a)\Mid\Tcal^{\star,D}_hf^{(k)}_{h+1}(x,a))$ and $\Delta_k:=\sum_{h=1}^H\EE_{\pi^k,x_{1,k}}[\delta_{h,k}(x_h,a_h)]$.
We now decompose the regret into two parts.
\begin{align*}
    &\sum_k V^{\pi^k}(x_{1,k})-V^\star(x_{1,k})
    \\&\leq \sum_kV^{\pi^k}(x_{1,k})-\min_a \bar f_1^{(k)}((x_{1,k}),a) \tag{Optimism}
    \\&=\sum_k\sum_{h=1}^H\EE_{\pi^k,x_{1,k}}\bracks{ \Tcal_h^{\pi^k}\bar f^{(k)}_{h+1}(x_h,a_h)-\bar f^{(k)}_h(x_h,a_h) } \tag{PDL}
    \\&=\sum_k\sum_{h=1}^H\EE_{\pi^k,x_{1,k}}\bracks{ \overline{\Tcal_h^{\pi^k}f^{(k)}_{h+1}}(x_h,a_h)-\bar f^{(k)}_h(x_h,a_h) }
    \\&\leq \sum_{h,k} 4\sqrt{\EE_{\pi^k,x_{1,k}}[\Var(f^{(k)}_h(x_h,a_h))]\cdot \EE_{\pi^k,x_{1,k}}[\delta_{h,k}(x_h,a_h)]} + 5\EE_{\pi^k,x_{1,k}}[\delta_{h,k}(x_h,a_h)] \tag{\cref{eq:var-key-ineq2}}
    \\&\leq \sum_{h,k} 4\sqrt{\prns{\Var(Z^{\pi^k}(x_{1,k}))+12H^2\Delta_k}\cdot \EE_{\pi^k,x_{1,k}}[\delta_{h,k}(x_h,a_h)]} + 5\EE_{\pi^k,x_{1,k}}[\delta_{h,k}(x_h,a_h)] \tag{\cref{eq:mainPaperChangeOfVarianceLemma}}
    \\&\leq \sum_k 4\sqrt{\prns{2e\Var(Z^{\pi^k}(x_{1,k}))+12H^2\Delta_k}\cdot H\Delta_k} + 5\Delta_k \tag{Cauchy-Schwarz}
    \\&\leq \sum_k 4\sqrt{2e\Var(Z^{\pi^k}(x_{1,k})) H\Delta_k} + (4\sqrt{12}+5)H^{1.5}\Delta_k
    \\&\leq 4\sqrt{2e\sum_k \Var(Z^{\pi^k}(x_{1,k})) H\sum_k\Delta_k} + (4\sqrt{12}+5)H^{1.5}\sum_k\Delta_k.
\end{align*}
The final step is to bound $\sum_k\Delta_k$, which is the same as in \citep{wang2023the}.
In particular, if \textsc{UAE}=\textsc{False}, then $\sum_k\Delta_k\leq cH \op{dim}_{\ell_1, DE}(1/K)\beta\log(K)$.
If \textsc{UAE}=\textsc{True}, then $\sum_k\Delta_k\leq cAH \op{dim}_{\ell_1, DE}(1/K)\beta\log(K)$.
This concludes the proof.
\end{proof}

\subsection{Bounding Q-type distributional Eluder in Linear MDPs}
Recall the Linear MDP definition \citep{jin2020provably}.
\begin{definition}[Linear and Low-Rank MDP]\label{def:linear-mdp}
A transition model $P_h:\Xcal\times\Acal\to\Delta(\Xcal)$ has rank $d$
if there exist features $\phi_h^\star:\Xcal\times\Acal\to\RR^d,\mu_h^\star:\Xcal\to\RR^d$ such that
$P_h(x'\mid x,a) = \phi_h^\star(x,a)^\top \mu_h^\star(x')$ for all $x,a,x'$.
Also, assume $\max_{x,a}\|\phi_h^\star(x,a)\|_2\leq 1$ and $\|\int g\diff\mu_h^\star\|_2\leq\|g\|_\infty \sqrt{d}$ for all functions $g:\Xcal\to\RR$.
The MDP is called low-rank if $P_h$ is low-rank for all $h\in[H]$.
The MDP is called linear if $\{\phi_h^\star\}_{h\in[H]}$ is known.
\end{definition}

Consider the following linear function class:
\begin{align}
    \mathcal{F}_h^{\text{lin}}=\bigg\{&
	f(z\mid x,a)=\big\langle\phi^\star(x,a),w(z)\big\rangle
	\text{\;s.t.\;} w:[0,1]\rightarrow\mathbb{R}^d,
        \max_z \|w(z)\|_2\leq \alpha\sqrt{d}
	\text{\; and\; }
\max_{x,a,z}\big\langle\phi^\star(x,a),w(z)\big\rangle\leq \alpha
	\bigg\}, \label{eq:linear-mdp-f-class}
\end{align}
\citet{wang2023the} showed two nice facts about $\Fcal^{\text{lin}}$.
First, it satisfies Bellman Completeness (\cref{asm:bellman-completeness}). Moreover, under the assumption that costs are discretized into a uniform grid of $M$ points, this class's bracketing entropy is $\wt\Ocal(dM)$.
Note that discretization is necessary to bound the statistical complexity of $\mathcal{F}^{\text{lin}}$ and is also common in practice, \eg, C51 \citep{bellemare2017distributional} and Rainbow \citep{hessel2018rainbow} both set $M=51$, which works well in Atari; also the optimal policy's performance in the discretized MDP can also be bounded by the discretization error \citep{wang2023near}.

We now show a new fact about $\mathcal{F}^{\text{lin}}$.
If we further assume that per-step cost and cost-to-go distributions have minimum mass $\eta_{\min}>0$ on each element of its support, then we can bound the appropriate $Q$-type distributional eluder dimension for linear MDPs as $\wt\Ocal(d\eta_{\min}^{-1}\log(1/\eps))$. This is formalized in the following assumption.
\begin{assumption}\label{asm:minimum-mass-linear-mdp}
For all $f\in\Fcal^{\text{lin}}$ and $h\in[H]$, if $f_h(z\mid x,a)=\Tcal^{\star,D}_hf_{h+1}(z\mid x,a)$, then $f_h(z\mid x,a)+\Tcal^{\star,D}_hf_{h+1}(z\mid x,a)\geq \eta_{\min}$.
\end{assumption}
If cost-to-go and per-step cost distributions have a minimum mass, then this assumption is satisfied.

\begin{theorem}
Suppose the MDP is a linear MDP and \cref{asm:minimum-mass-linear-mdp}.
Fix any $h\in[H]$ and define
\begin{align*}
    &\Psi_h = \braces{(x,a)\mapsto D_\triangle(f_h(x,a)\Mid \Tcal_h^{\star,D}f_{h+1}(x,a)):f\in\Fcal^{\normalfont\text{lin}}_h},
    \\&\Dcal_h = \braces{(x,a)\mapsto d^\pi_h(x,a):\pi\in\Pi}.
\end{align*}
Then, $\dim_{\ell_1\textsf{DE}}(\Psi_h, \Dcal_h, \eps) \leq \Ocal(d\eta_{\min}^{-1}\log(dM/(\eta_{\min}\eps)))$.
\end{theorem}
\begin{proof}
Fix any $h$.
Suppose $(d^{(k)},f^{(k)})_{k\in[T]}$ is any sequence such that for all $k\in[T]$, $d^{(k)}\in\Dcal_h$, $f^{(k)}\in\Psi_h$ and $(d^{(k)},f^{(k)})$ is $(\eps,\ell_1)$-independent of its predecessors. By definition, the largest possible $T$ is the eluder dimension of interest, so we now proceed to bound $T$.

For any $k$, since $f^{(k)}\in\Psi_h$, there exists $w^{(k)},v^{(k)}:[0,1]\to\RR^d$ satisfying normalization constraints of \cref{eq:linear-mdp-f-class} such that $f^{(k)}(x,a) = D_\triangle(z\mapsto \phi^\star_h(x,a)^\top w^{(k)}(z) \Mid z\mapsto \phi^\star_h(x,a)^\top v^{(k)}(z))$. Note that $v^{(k)}$ exists by Bellman completeness of $\Fcal^{\text{lin}}_h$.

Now we simplify the $D_\triangle$ term with the assumption: for any $k$,
\begin{align*}
    \EE_{d^{(k)}}D_\triangle(f_h^{(k)}(x,a)\Mid \Tcal^{\star,D}_h f_{h+1}^{(k)}(x,a))
    &=\EE_{d^{(k)}}\sum_z \frac{\prns*{f_h^{(k)}(z\mid x,a)-\Tcal^{\star,D}_h f_{h+1}^{(k)}(z\mid x,a)}^2}{f_h^{(k)}(z\mid x,a)+\Tcal^{\star,D}_h f_{h+1}^{(k)}(z\mid x,a)}
    \\&\leq \eta_{\min}^{-1} \EE_{d^{(k)}}\sum_z \prns*{\phi^\star_h(x,a)^\top (w^{(k)}(z)-v^{(k)}(z))}^2 \tag{\cref{asm:minimum-mass-linear-mdp}}
    \\&\leq \eta_{\min}^{-1} \EE_{d^{(k)}}\|\phi^\star_h(x,a)\|_{\Sigma_k^{-1}}^2 \cdot \sum_z \|w^{(k)}(z)-v^{(k)}(z)\|_{\Sigma_k}^2, \tag{CS}
\end{align*}
where $\Sigma_k:=\sum_{i<k}\EE_{d^{(i)}}[\phi^\star_h(x_h,a_h)\phi^\star_h(x_h,a_h)^\top] + \lambda I$ and $\lambda>0$ will be set soon.
For the second factor,
\begin{align*}
    \sum_z \|w^{(k)}(z)-v^{(k)}(z)\|_{\Sigma_k}^2
    &= \sum_z \sum_{i<k}\EE_{d^{(i)}}\prns{\phi^\star_h(x,a)^\top (w^{(k)}(z)-v^{(k)}(z))}^2 + M\lambda d
    \\&\leq \sum_{i<k}\EE_{d^{(i)}}\prns{\sum_z\abs{\phi^\star_h(x,a)^\top (w^{(k)}(z)-v^{(k)}(z))}}^2 + M\lambda d
    \\&\leq \sum_{i<k}\EE_{d^{(i)}}D_{\triangle}(f_h^{(k)}(x,a)\Mid \Tcal^{\star,D}_h f_{h+1}^{(k)}(x,a)) + M\lambda d \tag{$D_{TV}^2\leq D_\triangle$}
    \\&\leq \eps + M\lambda d \tag{$(\eps,\ell_1)$-independent sequence}
    \\&= 2\eps. \tag{set $\lambda=\eps/(dM)$}
\end{align*}
Thus, we have shown that
\begin{align*}
    T\eps
    &< \sum_k \EE_{d^{(k)}}D_\triangle(f_h^{(k)}(x,a)\Mid \Tcal^{\star,D}_h f_{h+1}^{(k)}(x,a)) \tag{$(\eps,\ell_1)$-independent sequence}
    \\&\leq \eta^{-1}_{\min}\sum_k \EE_{d^{(k)}}\|\phi^\star_h(x,a)\|_{\Sigma_k^{-1}}^2 \cdot 2\eps
    \\&\leq 2\eta^{-1}_{\min}\eps\cdot d\log(1+TM/\eps^2),
\end{align*}
where we used elliptical potential in the last step \citep[Lemma 19 \& 20]{uehara2021representation}, which is applicable since $\EE_{d^{(k)}}\|\phi_h^\star(x,a)\|_{\Sigma_k^{-1}}^2=\EE_{d^{(k)}}\phi_h^\star(x,a)^\top \Sigma_k^{-1}\phi_h^\star(x,a)=\Tr(\EE_{d^{(k)}}[\phi_h^\star(x,a)\phi_h^\star(x,a)^\top] \Sigma_k^{-1})$.
Thus, \citep[Lemma 19 \& 20]{uehara2021representation} implies that
\begin{equation*}
    T<2\eta^{-1}_{\min}d\log(1+TM/\eps^2),
\end{equation*}
which finally implies,
\begin{equation*}
    T \leq 12\eta^{-1}_{\min}d\log(1+2\eta^{-1}_{\min}dM/\eps^2),
\end{equation*}
by \citep[Lemma G.5]{wang2023the}.

\end{proof}

\section{Proofs for Offline RL}

\secondOrderOfflineRL*
\begin{proof}[Proof of \cref{thm:second-order-offline-rl}]
As noted by \citep[Proof of Theorem 6.1]{wang2023the}, the confidence set construction guarantees two facts w.p. $1-\delta$:
\begin{enumerate}[label=(\roman*)]
    \item Pessimism: for all $\pi$, $V^\pi\leq\bar f^\pi_1(x_1,\pi)$ (since $Z^\pi\in\Fcal_\pi$); and
    \item Small-generalization error: for all $\pi$ and $h$, $\EE_{\nu_h}[D_\triangle(f^\pi_h(x,a)\Mid \Tcal^{\pi,D}_h f^\pi_{h+1}(x,a))]\leq c\beta N^{-1}$ for some universal constant $c$.
\end{enumerate}
Let $\delta^\pi_h(x,a):=D_\triangle(f^\pi_h(x,a)\Mid\Tcal^{\pi,D}_hf^\pi_{h+1}(x,a))$ and $\Delta^\pi:=\sum_{h=1}^H\EE_{\pi}[\delta^\pi_h(x_h,a_h)]$.
We now bound the performance difference between $\wh\pi$ and $\wt\pi$:
\begin{align*}
    V^{\wh\pi}-V^{\wt\pi}
    &\leq \bar f^{\wh\pi}_1(x_1,\wh\pi) - V^{\wt\pi} \tag{Pessimism}
    \\&\leq \bar f^{\wt\pi}_1(x_1,\wt\pi) - V^{\wt\pi} \tag{Defn of $\wh\pi$}
    \\&=\sum_{h=1}^H\EE_{\wt\pi}\bracks{ \prns{\bar f^{\wt\pi}_h - \Tcal^{\wt\pi}_h \bar f^{\wt\pi}_{h+1}}(x_h,a_h) } \tag{PDL \cref{lem:performance-difference}}
    \\&\leq\sum_{h=1}^H4\sqrt{ \EE_{\wt\pi}[\Var(f^{\wt\pi}_h(x_h,a_h))]\cdot\EE_{\wt\pi}[\delta^{\wt\pi}_h(x_h,a_h)] } + 5\EE_{\wt\pi}[\delta^{\wt\pi}_h(x_h,a_h)] \tag{\cref{eq:var-key-ineq2}}
    \\&\leq\sum_{h=1}^H4\sqrt{ \prns{2e\Var(Z^{\wt\pi}) + 12H^2\Delta^{\wt\pi}}\cdot\EE_{\wt\pi}[\delta^{\wt\pi}_h(x_h,a_h)] } + 5\EE_{\wt\pi}[\delta^{\wt\pi}_h(x_h,a_h)] \tag{\cref{eq:mainPaperChangeOfVarianceLemma}}
    \\&\leq 4\sqrt{ \prns{2e\Var(Z^{\wt\pi}) + 12H^2\Delta^{\wt\pi}}\cdot H\Delta^{\wt\pi} } + 5\Delta^{\wt\pi} \tag{Cauchy-Schwarz}
    \\&\leq 4\sqrt{2e \Var(Z^{\wt\pi}) H \Delta^{\wt\pi}} + (4\sqrt{12}+5)H^{1.5}\Delta^{\wt\pi}.
\end{align*}
Finally, bound $\Delta^{\wt\pi}$ by change of measure and the generalization bound of MLE (fact (ii)):
\begin{equation*}
    \Delta^{\wt\pi}\leq C^{\wt\pi}\sum_{h=1}^H\EE_{\nu_h}[\delta^{\wt\pi}_h(x_h,a_h)]\leq C^{\wt\pi}H\cdot c\beta N^{-1}.
\end{equation*}
Therefore,
\begin{equation*}
    V^{\wh\pi}-V^{\wt\pi} \leq \Ocal\prns{ H\sqrt{\frac{C^{\wt\pi}\Var(Z^{\wt\pi})\beta}{N}} + \frac{H^{2.5}C^{\wt\pi}\beta}{N} }.
\end{equation*}
\end{proof}

\end{document}